\documentclass[nohyperref]{article}

\usepackage{microtype}
\usepackage{graphicx}
\usepackage{subfigure}
\usepackage{booktabs} 

\usepackage{hyperref}

\usepackage[accepted]{icml2023}

\usepackage{amsmath}
\usepackage{amssymb}
\usepackage{mathtools}
\usepackage{amsthm}
\usepackage{multirow}

\usepackage[capitalize,noabbrev]{cleveref}

\theoremstyle{plain}
\newtheorem{theorem}{Theorem}[section]

\newtheorem{lemma}[theorem]{Lemma}

\theoremstyle{definition}
\newtheorem{definition}[theorem]{Definition}

\theoremstyle{remark}

\usepackage[textsize=tiny]{todonotes}

\icmltitlerunning{MultiAdam: Parameter-wise Scale-invariant Optimizer for Physics-informed Neural Networks}

\begin{document}

\twocolumn[
\icmltitle{MultiAdam: Parameter-wise Scale-invariant Optimizer for \\ Multiscale Training of Physics-informed Neural Networks}



\icmlsetsymbol{equal}{*}

\begin{icmlauthorlist}
\icmlauthor{Jiachen Yao}{equal,thu,zl}
\icmlauthor{Chang Su}{equal,thu,zl}
\icmlauthor{Zhongkai Hao}{thu}
\icmlauthor{Songming Liu}{thu}
\icmlauthor{Hang Su}{thu}
\icmlauthor{Jun Zhu}{thu}
\end{icmlauthorlist}

\icmlaffiliation{thu}{Dept. of Comp. Sci. \& Tech., Institute for AI, BNRist Center, Tsinghua-Bosch Joint ML Center, THBI Lab, Tsinghua University}
\icmlaffiliation{zl}{Zhili College, Tsinghua University}

\icmlcorrespondingauthor{Jun Zhu}{dcszj@tsinghua.edu.cn}

\icmlkeywords{Physics-informed Neural Network, Optimizer, Multitask Learning, Scientific Machine Learning}

\vskip 0.3in
]



\printAffiliationsAndNotice{\icmlEqualContribution} 

\begin{abstract}

Physics-informed Neural Networks (PINNs) have recently achieved remarkable progress in solving Partial Differential Equations (PDEs) in various fields by minimizing a weighted sum of PDE loss and boundary loss.
However, there are several critical challenges in the training of PINNs, including the lack of theoretical frameworks and the imbalance between PDE loss and boundary loss.
In this paper, we present an analysis of second-order non-homogeneous PDEs, which are classified into three categories and applicable to various common problems.
We also characterize the connections between the training loss and actual error, guaranteeing convergence under mild conditions. 
The theoretical analysis inspires us to further propose MultiAdam, a scale-invariant optimizer that leverages gradient momentum to parameter-wisely balance the loss terms.
Extensive experiment results on multiple problems from different physical domains demonstrate that our MultiAdam solver can improve the predictive accuracy by 1-2 orders of magnitude compared with strong baselines.
\end{abstract}

\section{Introduction}
\label{intro}

Partial Differential Equations (PDEs) are important topics in applied mathematics, with a wide range of applications in various fields. The traditional approach to solving PDEs involves utilizing numerical techniques, such as the finite difference methods \cite{grossmann2007numerical} and finite element methods \cite{bathe2006finite}. Nevertheless, numerical methods may generate unrealistic predictions for specific scientific problems, and it is hard for these methods to handle PDEs in high dimensions \cite{zhu2019physics}. Therefore, it has attracted an increasing amount of attention to combine machine learning techniques for solving PDEs. Physics-informed Neural Network (PINN)~\cite{raissi2019physics} is one of the representative approaches that approximate solutions by training neural networks to minimize a weighted sum of PDE loss and boundary loss --- the former is induced from differential equations while the latter is induced from boundary and initial conditions.
PINN has shown its effectiveness in various sophisticated cases, which has been applied in various fields including fluids mechanics \cite{raissi2020hidden, sun2020surrogate}, and bio-engineering \cite{sahli2020physics, kissas2020machine}.

However, the vanilla PINN still suffers from some challenges during training \cite{hao2022physics}. One main challenge is the gap between PINN's loss function and the actual performance, which is often characterized by an absolute error.
In practical scenarios, certain loss terms, such as the PDE loss, might surpass others (e.g., boundary loss) by several orders of magnitude, consequently dominating the training process. This scenario can lead to situations where a reduced training objective—defined as a weighted sum of losses—does not necessarily yield a better approximation of the true solution~\cite{peng2020accelerating}. Our observations suggest that a key factor contributing to this challenge lies in the improper scaling of the PDE's domain. The scale of the domain can significantly affect PDE losses, especially when the PDE is not invariant to scaling—an occurrence that is quite common (see Theorem~\ref{thm:scaling}).
Specificaly, the scaling leads to two concrete issues:
\begin{enumerate}
    \item Due to the imbalance between PDE loss and boundary loss, the conventional optimizers like SGD and Adam might not sufficiently train the PINN model, motivating the development of more effective solvers. 
    \item  Given the observed discrepancy between the PINN's loss function and its actual performance, it becomes crucial to reevaluate the well-posedness of the optimization objective.
\end{enumerate}
To address the first issue, much work has focused on adjusting the relative importance of different loss terms by reweighting. Some works~\cite{wight2020solving, elhamod2022cophy} use manual hyper-parameters to adjust the weights.
However, these non-adaptive methods depend on empirical conclusions, which can lead to sub-optimal results. More research works focus on adaptively balancing PINN losses. For example,  \cite{wang2021understanding} designed a learning rate annealing algorithm using statistics of back-propagated gradients. \cite{wang2022and} proposed another method to adjust weights from the perspective of Neural Tangent Kernel (NTK). In \cite{bai2022physics}, the loss function is modified using the Least Squares Weighted Residual (LSWR) method. Nevertheless, these methodologies primarily concentrate on modifying loss functions, implying that they consider the effect on parameters as a whole. As such, they might overlook the impact of domain scaling on individual parameters of the model.

In this paper, we aim to address the above issues to effectively train PINNs. Specificially, we first present a theoretical error analysis of loss functions for different types of PDEs under mild conditions. This analysis provides a connection between the loss function and the actual performance of the model by bounding the $L^\infty$ error with its PDE loss and boundary loss. This new error boundary not only ensures convergence towards the ground truth under a sufficiently low loss but also serves as an optimization objective as its minimization enables the neural network to approach the true solution more effectively. \cite{wang2022is}'s work supports that $L^\infty$ loss is a better choice that $L^2$ loss.

Building on the upper-bound error, we propose a scale-invariant optimizer, MultiAdam. MultiAdam leverages the observation that the second momentum of Adam acts as an excellent indicator of the gradient scale. We categorize losses of different scales into separate groups, maintaining the second momentum individually for each group. This momentum is subsequently utilized to re-scale the gradients, aligning them to a nearly identical scale. Extensive experiments demonstrate that the MultiAdam optimizer is robust against unbalanced losses and is effective in various complex PDEs across different domain scales. Moreover, MultiAdam exhibits remarkable stability and a high convergence rate under these conditions.

The rest of the paper is organized as follows. In section \ref{related-work}, we briefly review existing variants of PINNs, especially reweighting techniques. In section \ref{section:pre-knowledge}, we go over the original PINN model and Adam optimizer. Section \ref{section:method} introduces the effect of domain scaling on PINN losses using an example of 2D Poisson's equation, followed by an introduction to our new optimizer MultiAdam. Then we provide a theoretical analysis on error bounds for PINNs and show the connection between the existing problem and MultiAdam. 
Section \ref{section:experiments} presents numerical experiments and evaluates MultiAdam using a range of representative benchmark examples. Finally, Section \ref{section:conclusion} encapsulates our findings and contributions.

\section{Related Work}
\label{related-work}


\textbf{Physics-informed Neural Networks} (PINNs) \cite{raissi2019physics} are capable of learning to represent the nonlinear relationship in dynamic systems and providing fast predictions \cite{karniadakis2021physics}. However, theoretical analysis for PINNs is typically insufficient. For some special equations such as Kolmogorov equations and Navier-Stokes equations, the total error can be estimated with regard to the training loss and network settings \cite{de2022error, de2022error2}. A more general result is attained on second-order elliptic equations, where the convergence of PINNs is proved \cite{shin2020error} and the $L^\infty$ error bound is given \cite{peng2020accelerating} under mild constraints. Yet it still remains unclear for many other PDE problems.
Also, analyzing the convergence and accuracy of PINNs is of tremendous challenge, especially for systems with multi-scale characteristics \cite{li2022dynamic}. 
PINNs are commonly optimized by Adam \cite{kingma2014adam} and L-BFGS \cite{liu1989limited}. However, they often reach ill situations when the scale and convergence rate of loss terms vary significantly \cite{hao2022physics}. 


\textbf{Reweighting techniques for PINNs} To correct the imbalance, a standard approach is the introduction of weights in the loss functions \cite{mcclenny2020self}. Currently, several adaptive reweighting methods have been proposed. \cite{wang2021understanding} designed a learning rate annealing algorithm using statistics of back-propagated gradients to mitigate the pathology. Neural Tangent Kernel also provides a novel perspective to adaptively adjust the weights \cite{wang2022and}. In \cite{bai2022physics}, the loss function is modified using the LSWR method to alleviate the biased training issue.


\textbf{Multitask learning methods} The PINN optimization can be regarded as a multitask learning problem since each equation and boundary condition is an individual objective. Therefore, it is also worthwhile to learn from multitask learning (MTL). GradNorm \cite{chen2018gradnorm} and PCGrad \cite{yu2020gradient} are two promising approaches along this line. GradNorm tunes gradient magnitudes based on the average gradient norm and the relative training rate of each task, while PCGrad projects the conflicting gradients onto the normal plane.

\section{Preliminaries}
\label{section:pre-knowledge}

\subsection{Physics-Informed Neural Networks}

The main objective for Physics-Informed Neural Networks (PINNs) is to solve a physical system using known physical laws and available data. Assume the system can be described by the following PDEs:
\begin{align}
\begin{aligned}
\label{formula: pinn_loss_def}
    &f(x;\frac{\partial u}{\partial x_1},\cdots,\frac{\partial u}{\partial x_d};\frac{\partial ^2 u}{\partial ^2 x_1^2}, \frac{\partial^2 u}{\partial x_1\partial x_2},\cdots;\lambda) = 0\\
    &B(u,x) = 0,\forall x\in \partial \Omega
\end{aligned}
\end{align}
where $f$ is the differential equation, $u$ is the solution to that equation, $\Omega$ is the domain and $\partial \Omega$ is the boundary of it. Moreover, $\lambda$ is an additional parameter and $B$ is the boundary condition.

To solve the physical system, PINNs use neural networks to approximate the solution of PDEs. In order to train a neural network meeting all the constraints in Eq. (\ref{formula: pinn_loss_def}), PINNs transform the equations into loss functions defined as follows:
\begin{align}\label{equ:pinn-loss}
\begin{aligned}
L_f(\theta, \lambda; T_f) &= \frac 1{|T_f|} \sum _{x\in T_f} \| f(x,\frac{\partial \hat u_\theta}{\partial x_1}, \cdots;\frac{\partial^2 \hat u_\theta}{\partial x_1^2},\cdots;\lambda)\|_2^2\\
L_b(\theta, \lambda; T_b) &= \frac 1{|T_b|} \sum _{x\in T_b} \| B(\hat u_\theta,x) \|_2^2
\end{aligned}
\end{align}
where $L_f$ is the residual loss for the PDE, and $L_b$ is the loss for boundary condition.
$u_\theta$ is the prediction by neural network with parameter $\theta$ and $T_f,T_b$ are sampling points.

The overall training objective of PINN is then defined as a weighted sum of the two losses:

\begin{align}\label{equ:pinn-loss-sum}
\begin{aligned}
    L(\theta, \lambda; T) = \ &w_f L_f (\theta, \lambda; T_f) + w_b L_b (\theta, \lambda ;T_b )
\end{aligned}
\end{align}

where $w_f, w_b$ are the non-negative weights for different losses. To effectively train a PINN, we have to optimize the two loss terms at the same time and make every loss as low as possible. Therefore, it is natural to treat it as a multitask learning problem.

\subsection{Adam Optimizer}

The Adaptive Momentum Estimation (Adam), as proposed by \cite{kingma2014adam}, is a commonly adopted optimization method for PINNs. It maintains the moving average of the squared gradient, known as the second momentum, to adjust the learning rate for each parameter. The specifics of this algorithm can be seen in Algorithm \ref{alg:ad}. Despite Adam's robust capability to minimize a single loss function for neural networks, it may struggle with handling multiple optimization objectives. Consequently, the network may fail to converge if the weights in Eq.~(\ref{equ:pinn-loss-sum}) are not appropriately configured. A detailed discussion on this matter will be provided in Section \ref{effect-scaling}.

\begin{algorithm}
\caption{Adam}\label{alg:ad}
\begin{algorithmic}[1]
	\REQUIRE learning rate $\gamma$, betas $\beta_1,\beta_2$, max epoch $M$, objective function $f(\theta)$
	\FORALL{$t=1$ to $M$}
	        \STATE $g_{t} \leftarrow \nabla _\theta f(\theta_{t-1})$
	        \STATE $m_{t} \leftarrow \beta_1 m_{t-1} + (1-\beta_1)g_t$
	        \STATE $v_{t} \leftarrow \beta_2 v_{t-1} + (1-\beta_2)g_t^2 $
	        \STATE $\hat m_{t} \leftarrow m_{t} / (1-\beta_1^t)$
	        \STATE $\hat v_{t} \leftarrow v_{t} / (1-\beta_2^t)$
	    \STATE $\theta_t \leftarrow \theta_{t-1} - \gamma \hat m_{t,i} / (\sqrt{\hat v_{t,i}} + \varepsilon)$
	\ENDFOR
	\STATE \textbf{return} $\theta_t$
\end{algorithmic}
\end{algorithm}

\section{Method}
\label{section:method}

We now present our method in detail, starting with an analysis of the imbalance between the terms in the loss objective.

\subsection{The effect of domain scaling on loss balancing}
\label{effect-scaling}

We first observe that the PDE loss and boundary loss may be several orders of magnitude away in real cases, leading to a failure to approach the correct solution by the standard Adam optimizer.
One of the main reasons for the issue is the improper scaling of the domain. Most PDEs are not scaling invariant, which causes the change in domain to rescale PDE loss. The influence is characterized by the following theorem:
\begin{theorem}[Effect of scaling for homogeneous PDEs, Proof in Appendix \ref{pf:scaling}]
\label{thm:scaling}
Suppose $\Omega$ is the domain of a homogeneous PDE of $k$ order and $L^2$ loss is used for PINNs. Then, if we narrow the domain by $t$ times, the boundary loss will stay fixed while the PDE loss will be multiplied by $t^{2k}$.

\end{theorem}

We illustrate this with an example of Poisson's equation in a complex domain. The reference solution is depicted in Figure \ref{pic:poisson-case}, with the detailed setup available in Appendix \ref{detail:poisson}. In this case, we condense the original domain, which spans an $8\times 8$ square, by a factor of $8$, resulting in a $1\times 1$ square. As shown in Figure \ref{pic:loss-diff}, when training on the $8\times 8$ domain, the PDE loss and boundary loss do not significantly differ. However, when training on the $1\times 1$ domain, the PDE loss is nearly $8^4$ times larger than the boundary loss. This substantial discrepancy poses considerable challenges in training PINNs, as demonstrated in Figure \ref{pic:poisson-case}.

This example further exposes the gap between the loss function that PINN optimizes and its actual performance. In Figure \ref{pic:loss-error-gap}, we train PINN on the $1\times 1$ domain using two different settings—one incorporating manual reweighting of loss while the other not. In the absence of manual reweighting, PINN fails to approach the ground truth. Yet, its loss is lower than that of the reweighted scenario for the first $10000$ epochs, during which its $L^2$ relative error in relation to the ground truth is significantly higher compared to the reweighted scenario. This suggests that the loss optimized in PINN does not reliably represent the actual performance in this case.

\begin{figure*}[ht]
\vskip 0.2in
\begin{center}
\centerline{\includegraphics[width=2\columnwidth]{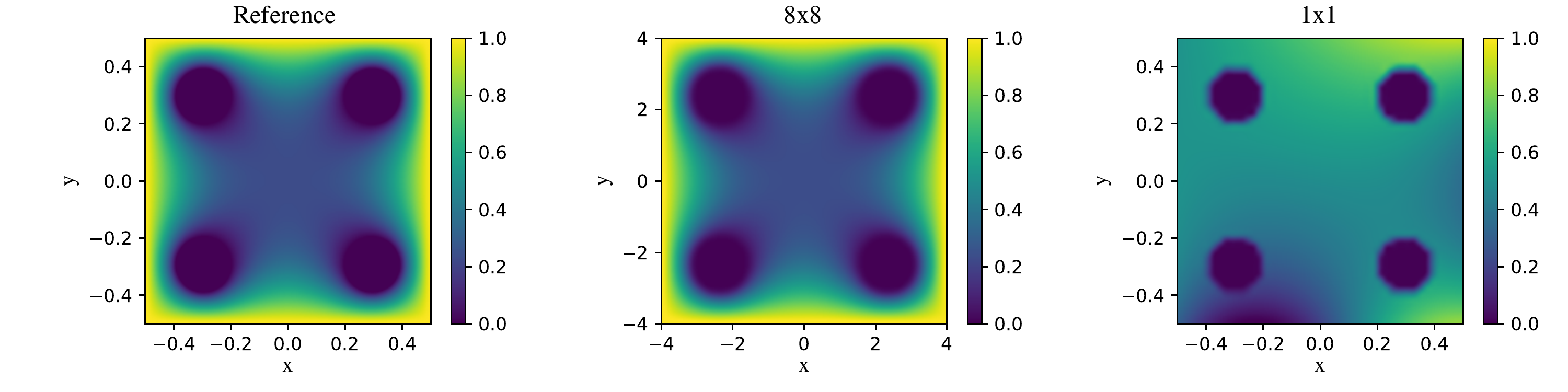}}
\caption{
The left image presents the reference solution for the case. The central image depicts the training result of the baseline PINN on an $8\times 8$ domain, while the right image showcases the same on a $1\times 1$ domain. It is evident that the model encounters difficulties in fitting the boundary condition when trained on the $1\times 1$ domain.
}
\label{pic:poisson-case}
\end{center}
\vskip -0.2in
\end{figure*}

\begin{figure}[ht]
\begin{center}
\centerline{\includegraphics[width=\columnwidth]{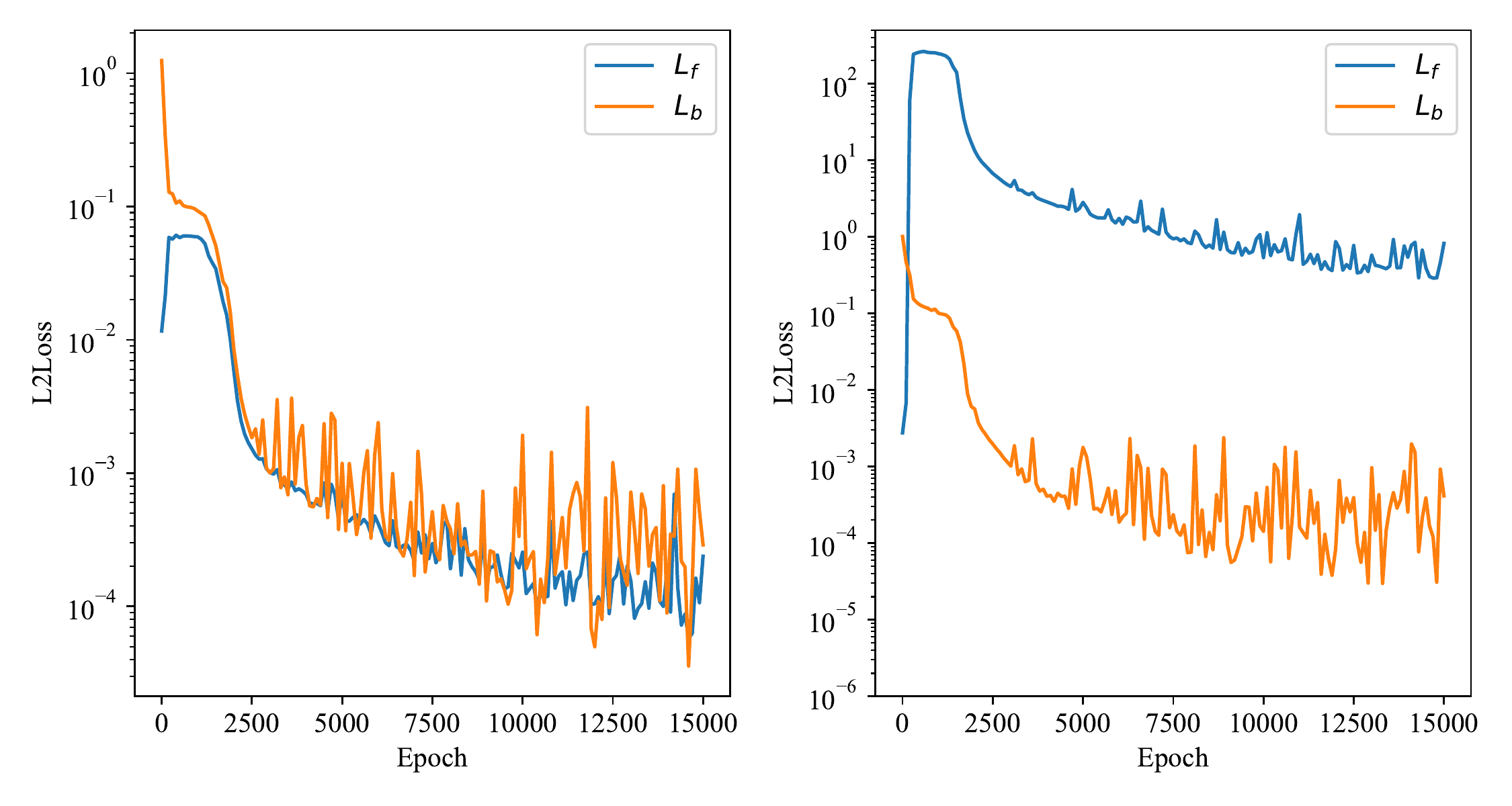}}
\caption{The loss curve of PINNs when solving Poisson equation on the $8\times 8$ domain and the $1\times 1$ domain using Adam optimizer (manually reweighted on $1\times 1$ case). $L_f, L_b$ are defined in equation (\ref{equ:pinn-loss}). While the losses are almost the same on $8\times 8$ case, they differ by several orders of magnitude on the $1\times 1$ case. }
\label{pic:loss-diff}
\end{center}
\vskip -0.2in
\end{figure}

\begin{figure}[ht]
\begin{center}
\centerline{\includegraphics[width=\columnwidth]{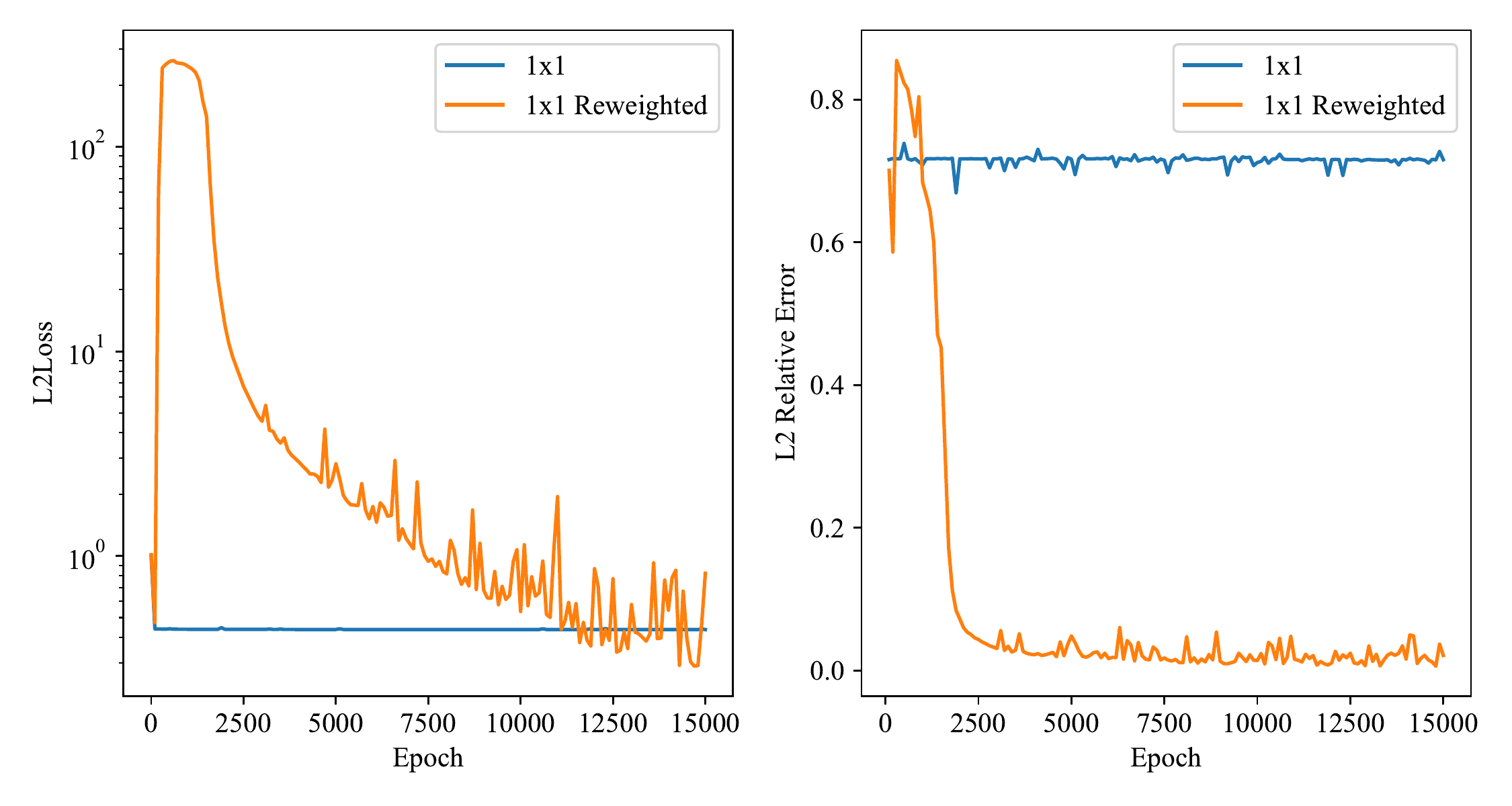}}
\caption{The left figure shows the sum of unweighted loss $L_f + L_b$ during training. The right figure shows the $L^2$ relative error between PINN's prediction and the ground truth. While the loss is lower in the unreweighted case, the prediction is worse off.}
\label{pic:loss-error-gap}
\end{center}
\vskip -0.2in
\end{figure}

\subsection{Error Analysis}

Considering the observed inconsistency between total loss and actual performance, we find it crucial to revisit the well-posedness of our objective function, i.e., we question whether optimization based on the loss indeed leads to improved solutions. We offer a theoretical examination of the relationship between loss and error. Given that the majority of PDEs employed across various disciplines do not exceed second order, and that linear ones are relatively prevalent and simpler to analyze, our study primarily concentrates on elliptic, parabolic, and select hyperbolic equations. These represent the majority of second-order linear PDEs \cite{strauss2007partial}. Based on the error bounds, we establish links between the losses in PINNs and the absolute error of the PINN output.

Specifically, we provide error bounds for three types of PDEs separately in the following theorems. Thanks to Theorem 2.1 and Corollary 2.2 in \cite{peng2020accelerating}, we can directly obtain the proof of Error bounds of PINNs on elliptic PDEs as follows: 
\begin{theorem}[Error bounds of PINNs on elliptic PDEs]
\label{thm:elliptic-1}
Suppose $\Omega \subset \mathbb{R}^d$ is a bounded domain, $\mathcal L$ is an elliptic operator and $\tilde u\in C^0(\overline \Omega)\cap C^2(\Omega)$ is a solution to the following PDE:
\begin{align} \label{equ:theorem-pde}
\begin{aligned}
\mathcal L[u](x) = f(x),\ \  &\forall x\in \Omega\\
u(x) =g(x), \ \ &\forall x\in \partial \Omega
\end{aligned}
\end{align}
If the output $u_\theta$ of the PINN with parameter $\theta$ satisfies:
\begin{align}
\begin{aligned}
&u_\theta \in C^0(\overline \Omega)\cap C^2 (\Omega)\\
&\sup _{x\in \partial \Omega}|u_\theta - \tilde u |<\delta _1\\
&\sup _{x\in \Omega}|\mathcal L[u_\theta ] - f|<\delta_2,
\end{aligned}
\end{align}
then the absolute error over $\Omega$ is upper-bounded:
\begin{equation}
    \sup _{x\in\Omega}|u_\theta -\tilde u|\le \delta_1 + C\delta_2.
\end{equation}
Here, $C$ is a constant depending only the operator $\mathcal{L}$ and the domain $\Omega$. If $\mathop{\mathrm{diam}} \Omega = d$, then $C$ is proportional to $e^{d}-1$ when $\mathop{\mathrm{diam}} \Omega$ changed.
\end{theorem}

And we further provide the Error bounds of PINNs on Parabolic PDEs and Hyperbolic PDEs in Theorem~\ref{thm:parabolic-2} and~\ref{thm:hyperbolic-2}, respectively. The detailed proof is included in the Appendix. 

\begin{theorem}[Error bounds of PINNs on Parabolic PDEs, proof in Appendix \ref{pf:parabolic-2}]
\label{thm:parabolic-2}
Suppose $\Omega \subset \mathbb{R}^d_x \times \mathbb{R}_t$ is a bounded domain, $\mathcal L$ is an parabolic operator and $\tilde u\in C^0(\overline \Omega)\cap C^2(\Omega)$ is a solution to the PDE in equation \ref{equ:theorem-pde}. If the output $u_\theta$ of the PINN with parameter $\theta$ satisfies:
\begin{align}
\begin{aligned}
&u_\theta \in C^0(\overline \Omega)\cap C^2 (\Omega)\\
&\sup _{x\in \partial \Omega}|u_\theta - \tilde u |<\delta _1\\
&\sup _{x\in \Omega}|\mathcal L[u_\theta ] - f|<\delta_2,
\end{aligned}
\end{align}
then the absolute error over $\Omega$ is upper-bounded:
\begin{equation}
\sup _{x\in\Omega}|u_\theta -\tilde u|\le C_1(\delta_1 + C\delta_2),
\end{equation}
where $C,C_1$ are constants depending only on $\Omega$ and $\mathcal L$. If $\mathop{\mathrm{diam}} \Omega = d$, then $C$ is proportional to $e^{\alpha d} - 1$ when $\mathop{\mathrm{diam}} \Omega$ changed.
\end{theorem}

\begin{theorem}[Error Bounds for PINNs on Hyperbolic PDEs, proof in Appendix \ref{pf:hyperbolic-2}]
\label{thm:hyperbolic-2}
Suppose $\Omega \subset \mathbb{R}_x\times \mathbb{R}_t^+$ is an admissible domain (defined in Appendix \ref{pf:hyperbolic-2}) and $\mathcal L$ is an hyperbolic operator satisfies the requirements in Appendix \ref{pf:hyperbolic-2}. If the PINN with parameter $\theta$ satisfies that:
\begin{align}
\begin{aligned}
&u_\theta \in C^1(\overline \Omega)\cap C^2 (\Omega)\\
&\sup _{x\in \partial \Omega}|u_\theta - \tilde u |<\delta _1\\
&\sup _{x\in \Omega}|\mathcal L[u_\theta ] - f|<\delta_2
\end{aligned}
\end{align}
Then, we have:
$$
\sup _{x\in \Omega}|u_\theta - \tilde u|\le \delta_1+C\delta_2
$$
where $C$ is constant depending only on $\Omega$ and $\mathcal {L}$. If $\mathop{\mathrm{diam}} \Omega = d$, then $C$ is proportional to $e^{\alpha d} - 1$ when $\mathop{\mathrm{diam}} \Omega$ changed.
\end{theorem}

We finally provide how to control the absolute error using PINNs' $L^2$ loss as 
\begin{theorem}[Control Absolute Error using PINNs' $L^2$ Loss, proof in Appendix \ref{pf:L2control}]
\label{thm:L2control}
Suppose the second-order PDE operator $\mathcal{L}$ and the PINN with parameter $\theta$ satisfy that:
\begin{equation}
    \sup _{x\in \Omega}|u_\theta - \tilde u|\le C_1(\sup_{x\in \partial \Omega} |u_\theta - \tilde u|+C \sup_{x\in \Omega} |\mathcal L[u_\theta] - f|)
\end{equation}
where $C, C_1$ are constants. Then, the error can be bounded by $L^2$ loss of the PINN:
\begin{equation}
    \|u_\theta - \tilde u\|_{L_\infty} \le C_2(\sqrt{L_b} + C \sqrt{L_f})
\end{equation}
where $C_2$ is constant depend on $C_1$ and selection of sampling points and base functions (used in proof). The detailed definition is in Appendix \ref{pf:L2control}.
\end{theorem}

Theorem \ref{thm:L2control} delineates the relationship between the loss of PINNs and the actual error. Although the unweighted sum of losses does not directly reflect the performance of PINNs, the introduction of appropriate weights to the losses can ensure a more accurate correspondence to error. This underlines the necessity of reweighting techniques for PINNs. Broadly, the more precise the estimate of $C$, the narrower the gap between the optimization objective and the actual error.

The theorem also illuminates the role of domain scaling. For all three types of PDEs, scaling the domain influences the constants $C$, changing proportionally to $e^{\alpha d} - 1$, where $d = \mathop{\mathrm{diam}} \Omega$ serves as an indicator of scale. This modification in $C$ subsequently affects the optimal weight of the two losses. Therefore, it is imperative for the model to account for the scale of the domain to properly adjust the loss weights.

Motivated by this understanding, we propose our MultiAdam optimizer. It maintains the second momentum of gradients for each group of losses, which is then used to adjust the scale of the update, effectively reweighting all loss terms. We found that gradient-based estimation can approximate the factor $C$, leading to enhanced accuracy.

\subsection{Algorithm}

Inspired by the analysis above, we introduce MultiAdam, a novel optimizer designed to better estimate the relative importance of losses.

Our motivation stems from two key observations. First, the Adam optimizer maintains estimates of both the first and second momentum, and these momentums tend to be relatively stable. Second, the second momentum effectively reflects the inherent difference between the scale of PDE loss and boundary loss. Utilizing the second momentum as weights allows the PDE loss and boundary loss to be normalized to a comparable scale.

The crux of MultiAdam lies in partitioning the PINN loss into several groups. Specifically, we segregate each PDE loss into a separate group, while all boundary losses are grouped together. We maintain the first and second momentum independently for each group, determining the update for every group in a manner akin to Adam. Lastly, we average the updates for each group and apply this as the final update to the network parameters.

The specific algorithm is outlined in Algorithm \ref{alg:ma}. We recommend the hyper-parameter settings as $\gamma=0.001,\beta_1=0.99,\beta_2=0.99$. The rationale behind these choices can be found in Appendix \ref{ablation}.

\begin{algorithm}
\caption{MultiAdam}\label{alg:ma}
\begin{algorithmic}[1]
	\REQUIRE learning rate $\gamma$, betas $\beta_1,\beta_2$, max epoch $M$, objective functions $f_1(\theta), f_2(\theta), \cdots, f_n(\theta)$
	\FORALL{$t=1$ to $M$}
	    \FORALL{$i=1$ to $n$}
	        \STATE $g_{t,i} \leftarrow \nabla _\theta f_i(\theta_{t-1})$
	        \STATE $m_{t,i} \leftarrow \beta_1 m_{t-1,i} + (1-\beta_1)g_{t,i}$
	        \STATE $v_{t,i} \leftarrow \beta_2 v_{t-1,i} + (1-\beta_2)g_{t,i}^2 $
	        \STATE $\hat m_{t,i} \leftarrow m_{t,i} / (1-\beta_1^t)$
	        \STATE $\hat v_{t,i} \leftarrow v_{t,i} / (1-\beta_2^t)$
	    \ENDFOR
	    \STATE $\theta_t \leftarrow \theta_{t-1} - \frac\gamma n \sum _{i=1}^n \hat m_{t,i} / (\sqrt{\hat v_{t,i}} + \varepsilon)$
	\ENDFOR
	\STATE \textbf{return} $\theta_t$
\end{algorithmic}
\end{algorithm}

The reason why we divide every PDE into separate groups is that different PDE has a different intrinsic scaling factor, leading to an imbalance within the same group. Conversely, all Dirichlet boundary losses are grouped together, as they are calculated by measuring the $L^2$ error on sampling points, which remains invariant to the scaling of the domain.

\section{Experiments}
\label{section:experiments}

In this section, we deploy our proposed MultiAdam optimizer on various benchmarks to evaluate its convergence and accuracy. Initially, we consider Poisson's equation, a two-dimensional second-order linear PDE. This serves to examine MultiAdam's efficacy in mitigating the imbalance of weights and achieving convergence. We also compare its weight estimation to the theoretically suggested weight, demonstrating its consistency across diverse domain scales. Subsequently, we apply this method to solve the non-linear elliptic-type Helmholtz equation, underscoring the efficiency of MultiAdam. Lastly, we assess the performance of our method against other techniques in solving time-dependent PDEs, such as the Burgers' equation. An ablation study on the selection of hyper-parameters is presented, which is relegated to Appendix \ref{ablation}.

We compare our method with a few strong baselines: 1) The Adam optimizer utilized by the original PINNs \cite{raissi2019physics} 2) The learning rate annealing (LRA) algorithm for PINNs~\cite{wang2021understanding} and 3) The adaptive weighting from the NTK perspective \cite{wang2022and}. Since PINNs involve the interplay of multiple loss terms from PDE and boundary conditions, some multi-task learning methods may be applied to PINNs. Here, we choose two well-known methods, i.e., 4) GradNorm \cite{chen2018gradnorm} and 5) PCGrad \cite{yu2020gradient}, to compare with.

\subsection{Poisson's equation}
\label{subsec:laplace-equation}
Poisson's equation is a useful elliptic partial differential equation in theoretical physics for calculating electric or gravitational fields \cite{wiki:Poisson's_equation}, taking the form:
\begin{equation}
    \Delta u=f
\end{equation}
In order to show the scale-invariant ability of MultiAdam, we consider two Poisson's systems, Poisson-8 and Poisson-1, which are actually examples presented in Section \ref{effect-scaling}. The Poisson-8 case is as Equation \ref{eqn:poi8} in Appendix \ref{detail:poisson}.
, while the Poisson-1 case just resizes the domain from $[-4,4]^2$ to $[-0.5,0.5]^2$.

As shown in Table \ref{poi-error}, MultiAdam is nearly invariant to the domain scaling and maintains an accurate estimate. For Poisson-8, NTK has the highest precision. However, in the Poisson-1 case, things have changed. Most of the optimizers, other than MultiAdam and NTK, fail to find the solution. MultiAdam performs the best while a significant downgrade (4.17\%) is observed on NTK. Overall, MultiAdam can easily handle the domain-scaling effect and keep good performance on both tests while others cannot.

\begin{table}[t]
\caption{Mean absolute error and relative $L^2$ error of different optimization methods on Poisson's equation. PCGrad runs into NaN due to numerical instability.}
\label{poi-error}
\vskip 0.15in
\begin{center}
\begin{small}
\begin{tabular}{lllll}
\hline
\multirow{2}{*}{Methods} & \multicolumn{2}{c}{Poisson-8} & \multicolumn{2}{c}{Poisson-1} \\
                         & Absolute   & Relative         & Absolute   & Relative         \\ \hline
Adam                     & 7.49E-03   & 2.63\%           & 2.98E-01   & 70.78\%          \\
LRA                      & 1.06E-02   & 4.67\%           & 6.48E-02   & 16.88\%           \\
NTK                      & \textbf{6.58E-03}   & \textbf{1.94\%}           & 2.21E-02   & 6.11\%          \\
GradNorm                 & 8.74E-03   & 2.34\%           & 2.94E-01   & 69.10\%          \\
PCGrad                   & N/A   & N/A  & 3.40E-01   & 77.84\%          \\ 
MultiAdam                & 1.10E-02   & 2.94\%           & \textbf{1.44E-02}   & \textbf{4.49\%}  \\ \hline
\end{tabular}
\end{small}
\end{center}
\vskip -0.1in
\end{table}

\subsubsection{Comparison of weight estimation}

To give a deeper understanding on why MultiAdam outperforms other methods when domain is changed, we compare the weights given by different reweighting algorithms with a theoretically suggested weight as summarized in the following theorem.
\begin{theorem}[Error bound of Poisson's equation, Proof in Appendix \ref{pf:laplaceerror}]
\label{thm:laplaceerror}
Let $\Omega$ be the domain described in section \ref{subsec:laplace-equation}, and $G: \Omega\times \Omega \to \mathbb{R}$ be the Green function of Poisson's equation. Denote $\hat u_\theta$ as the PINN output and $\tilde u$ the reference solution, then we have:
\begin{align}
\begin{aligned}
\|\hat u_\theta - \tilde u\|_{L^1} \le C_1 \sqrt{L_{f}} + C_2\sqrt{L_{b}} ,
\end{aligned}
\end{align}
where $L_{f}, L_{b}$ are losses of PINN and $C_1, C_2$ are constants by the Green function $G(x, \xi)$ as follows:
\begin{align}
\begin{aligned}
C_1 &= \int _{\Omega} \sqrt{|\Omega|\int _{\Omega}G^2(x,\xi) \text d \xi} \text d x\\
C_2 &= \int _{\Omega} \sqrt{|\partial \Omega|\int_{\partial\Omega}(\nabla_{\xi}G(x, \xi) \cdot \mathbf{n})^2} \text d x .
\end{aligned}
\end{align}
\end{theorem}

According to the above theorem, the best strategy to minimize $\| v\|_{L^1}$ is to minimize $\sqrt{ C_1^2 L_{f}} + \sqrt{C_2^2L_{b}}$. This implies the assignment of weight $C_1^2$ to the PDE loss and $C_2^2$ to the boundary loss.

Then we run MultiAdam for multiple times and record the norm of second momentum for the PDE loss group and boundary loss group separately. Since we use second momentum to rescale the gradients, its norm reflects how we scale the gradient as a whole. Therefore, in the following comparison, the norm of second momentum is used as our estimated weight for different losses.

For comparison purposes, we also incorporate two other reweighting techniques, LRA and NTK. By normalizing the weight on the boundary loss to 1, we can directly compare the normalized weight on the PDE loss and discern how the algorithms balance between different losses. We run the different methods three times, with the results displayed in Figure \ref{pic:poisson-weight-epoch}.

\begin{figure}[ht]
\vskip 0.2in
\begin{center}
\centerline{\includegraphics[width=\columnwidth]{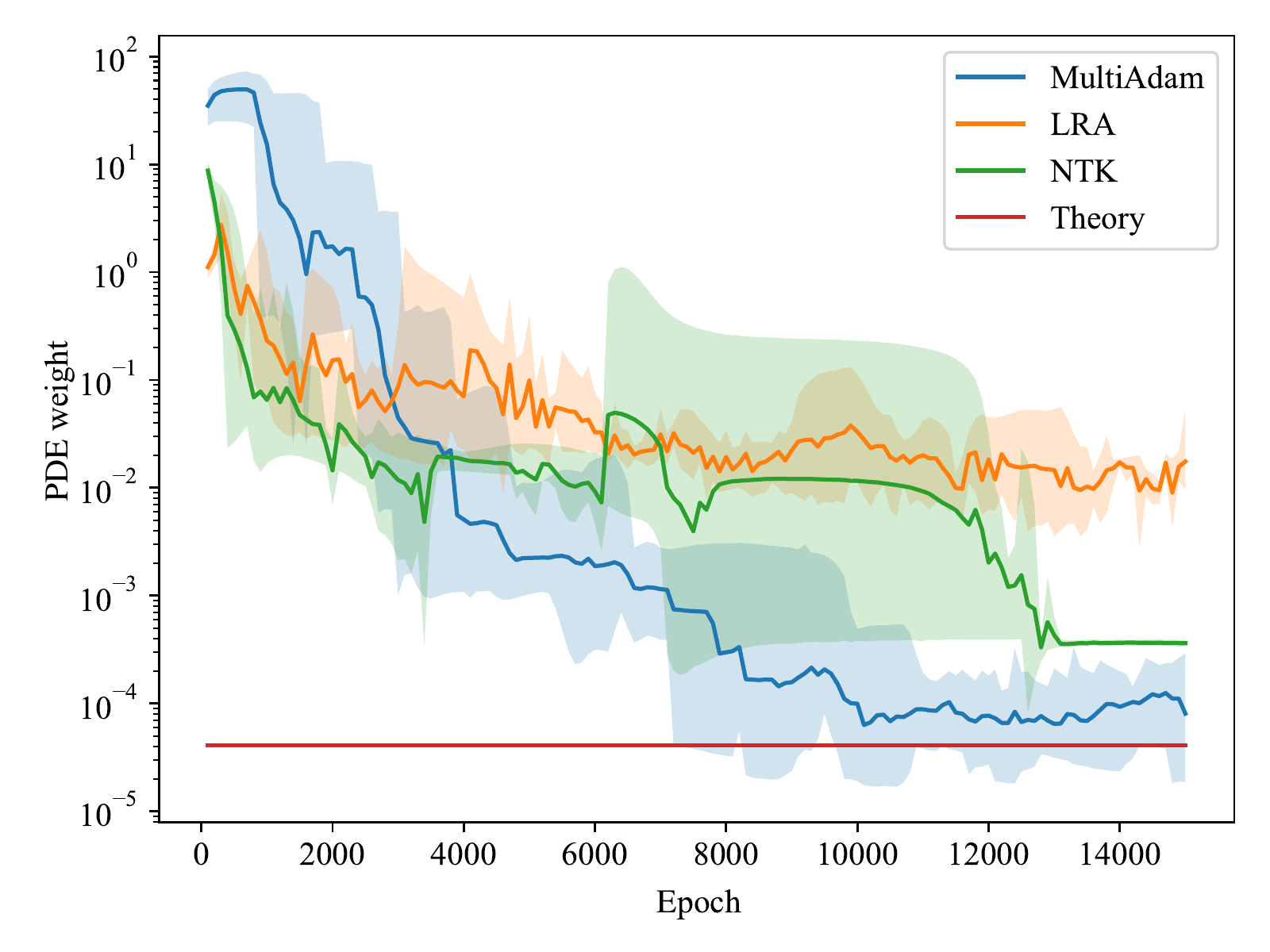}}
\caption{The comparison of normalized weight for PDE loss between MultiAdam, LRA, NTK and theoretical suggestion during training. The domain $\Omega$ lies in $[-0.5, 0.5]^2$. The estimation given by MultiAdam is closest to the theoretical suggestion.}
\label{pic:poisson-weight-epoch}
\end{center}
\vskip -0.2in
\end{figure}

We observe that the weight assigned by MultiAdam closely aligns with the theoretical prediction. This implies that MultiAdam accurately discerns the relative importance of different tasks, enabling it to balance the gradients of various groups and approximate the ground truth closely. It's worth noting that the slightly higher PDE weight, compared to the theoretical estimation, is attributed to the difficulty PINNs face in optimizing the PDE loss.

More crucially, MultiAdam successfully mirrors the growth trend of PDE weight under differnt scales. As depicted in Figure \ref{pic:poisson-weight-scale}, MultiAdam exhibits superior estimation in most scales compared to other methods. 
These results provide a support for MultiAdam's ability to handle problems under different scales.

\begin{figure}[ht]
\begin{center}
\centerline{\includegraphics[width=\columnwidth]{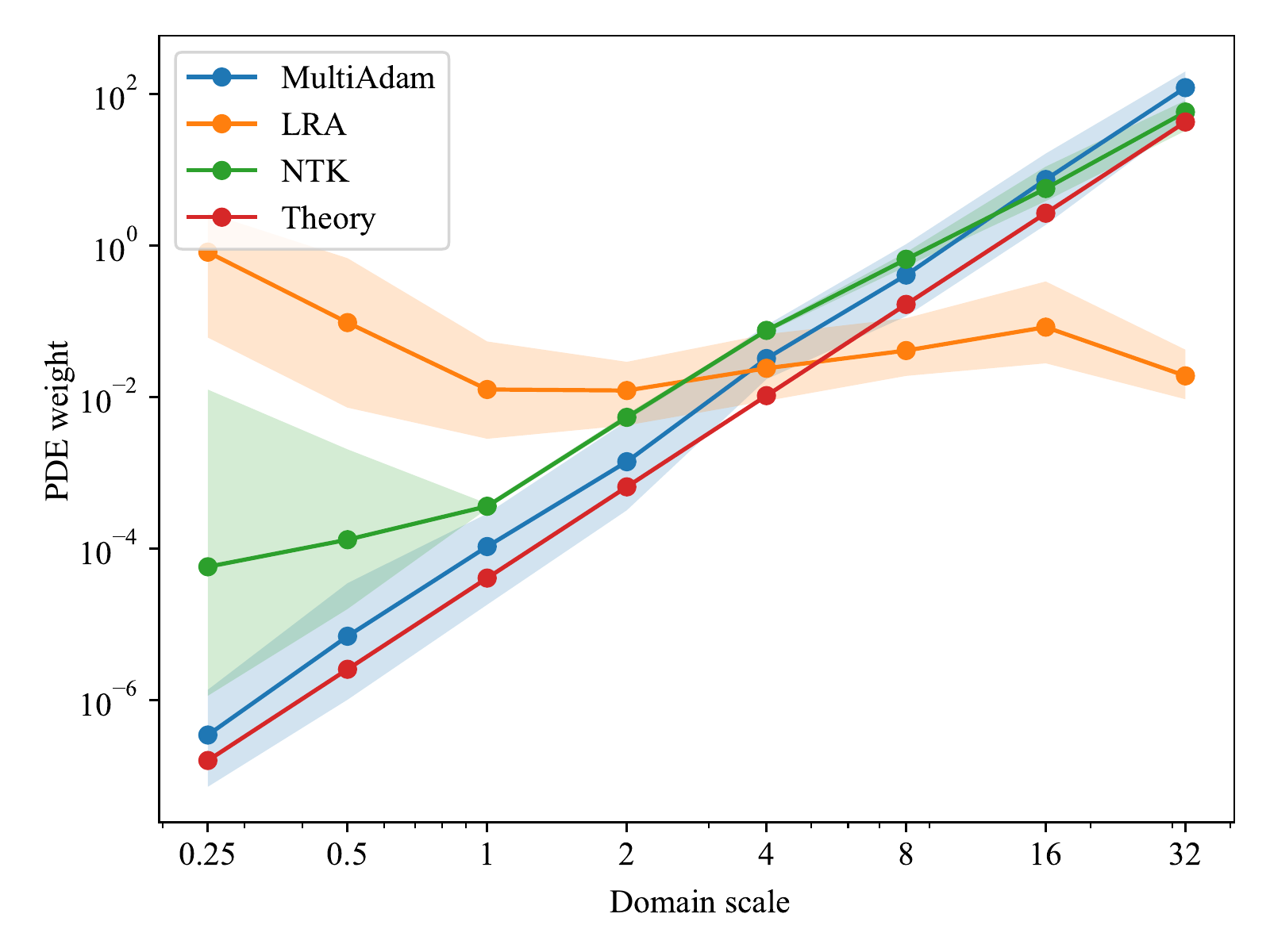}}
\caption{The comparison of normalized weight for PDE loss between MultiAdam, LRA, NTK and theoretical suggestion under different domain scales. When domain scale is $x$, we indicates that the domain $\Omega$ lies in $[-x/2, x/2]^2$}.
\label{pic:poisson-weight-scale}
\end{center}
\vskip -0.2in
\end{figure}

\subsubsection{Gradient pathology}

To further investigate the pathology of imbalanced gradients, we study the distribution of the gradients regarding the PDE residual and the boundary loss. The results are shown in Figure \ref{poi-grad}. We can see that MultiAdam can mitigate the gradient-vanishing problem in PINNs and effectively update parameters. The PDE gradients of the original PINNs are heavily concentrated around zero and barely can parameters be optimized, leading to stagnation. This observation is inline with \cite{wang2021understanding}'s work. By contrast, the PDE gradients of MultiAdam PINNs are more spread, thus more parameters can attain useful information, accelerating the overall optimization.

\begin{figure}[ht]
\vskip 0.2in
\begin{center}
\centerline{\includegraphics[width=\columnwidth]{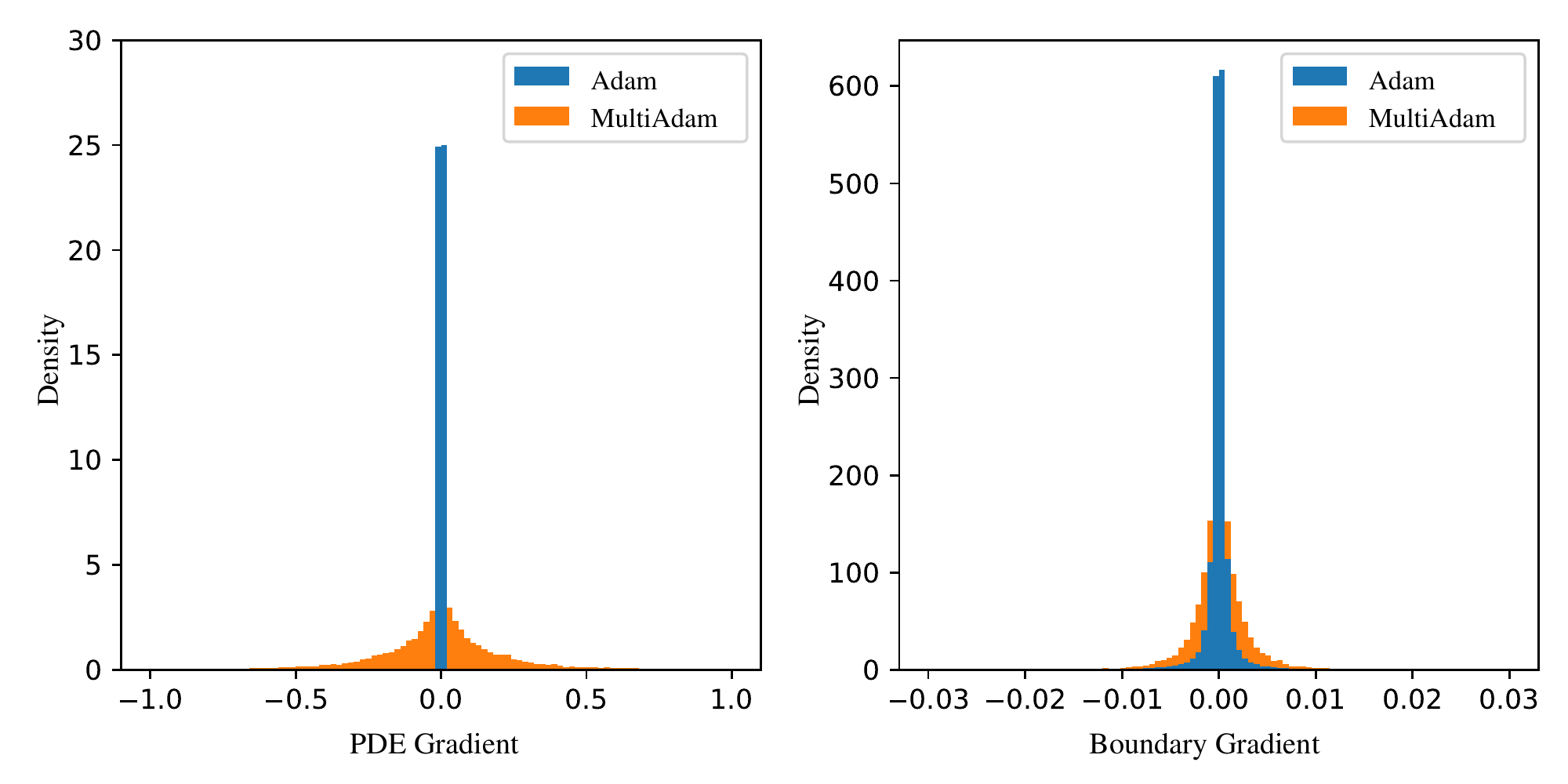}}
\caption{The distribution of the back-propagated gradients over different loss groups (PDE, boundary) at epoch 4000.}
\label{poi-grad}
\end{center}
\vskip -0.2in
\end{figure}

\subsection{Helmholtz equation}

The Helmholtz equation is a non-linear elliptic differential system representing a time-independent form of the wave equation. It appears in various fields of physics, including electromagnetic radiation, seismology, and acoustics \cite{wiki:Helmholtz_equation}. The Helmholtz equation is a good testbed to demonstrate the ability to cope with highly non-linear problems. Specifically, the equation takes the following form:
{
\begin{align}
\begin{aligned}
    & u_{xx}+u_{yy}+k^2u-f=0,~\forall x \in\Omega \\
    & u(x)=0,~\forall x\in \partial\Omega \\
    & \Omega=[-\frac{b}{2},\frac{b}{2}]^2 ,
\end{aligned}
\end{align}
}
where $k$ is a parameter. The initial-boundary value problem has exact solution $u(x,y)=\sin(a_ix)\sin(a_2x)$ when
\begin{equation}
    f(x,y)=(k^2-a_1^2\pi^2-a_2^2\pi^2)\sin(a_1\pi x)\sin(a_2\pi y)
\end{equation}

We consider two cases, $(k=1,a_1=a_2=1,b=1)$ and $(k=1,a_1=a_2=10,b=0.2)$, denoted as Helmholtz-1 and Helmholtz-0.2 respectively. Figure \ref{hel-heatmap} in Appendix \ref{app:hel} presents the reference solution.

From the perspective of both absolute error and relative error in Table \ref{hel-error}, MultiAdam achieves the highest accuracy among these techniques . It improves the relative $L^2$ error by roughly two orders of magnitude. After resizing the domain, MultiAdam does not suffer while the competitors do, which again demonstrates the robustness of our method against re-scaling.

\begin{table}[t]
\caption{Mean absolute error and relative $L^2$ error of different optimization methods on Helmholtz equation.}
\label{hel-error}
\vskip 0.15in
\begin{center}
\begin{small}
\begin{tabular}{lllll}
\hline
\multirow{2}{*}{Methods} & \multicolumn{2}{c}{Helmholtz-1} & \multicolumn{2}{c}{Helmholtz-0.2} \\
                         & Absolute    & Relative          & Absolute     & Relative           \\ \hline
Adam                     & 8.50E-02    & 22.46\%           & 3.45E-01     & 93.46\%               \\
LRA                      & 4.00E-03    & 1.11\%           & 1.65E-01     & 45.87\%              \\
NTK                      & 8.32E-02    & 21.76\%           & 5.05E-01     & \textgreater 100\%              \\
GradNorm                 & 6.15E-02    & 16.06\%           & 3.97E-01     & \textgreater 100\%              \\
PCGrad                   & 1.79E-02    & 4.80\%           & 8.67E-02     & 22.92\%              \\ 
MultiAdam       & \textbf{1.56E-03}    & \textbf{0.43\%}   & \textbf{3.23E-03}     & \textbf{0.87\%}    \\ \hline
\end{tabular}
\end{small}
\end{center}
\vskip -0.1in
\vspace{-1.3em}
\end{table}

\subsubsection{Rate of convergence}

\begin{figure}[ht]
\vskip 0.2in
\begin{center}
\centerline{\includegraphics[width=\columnwidth]{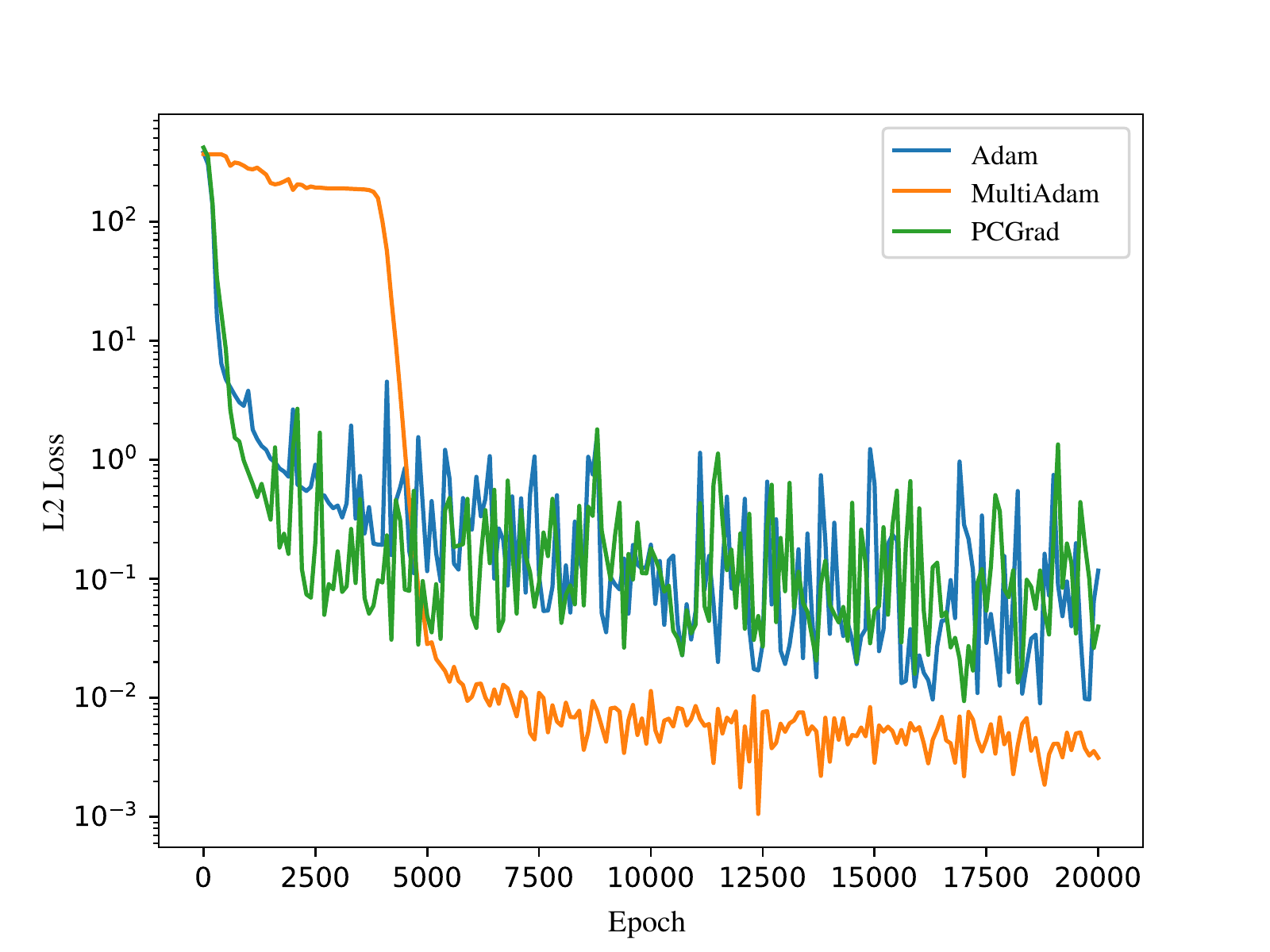}}
\caption{$L^2$ loss curves in the Helmholtz-1 case trained with Adam, MultiAdam or PCGrad.}
\label{hel-rate}
\end{center}
\vskip -0.2in
\end{figure}

We choose three representative algorithms, namely Adam, MultiAdam and PCGrad, to compare their convergence speeds from the perspective of $L^2$ loss curves. As shown in Figure \ref{hel-rate}, it is interesting to see that MultiAdam moves slow in the beginning phase (e.g., $<5000$ epochs), while can quickly converge to better solutions. The reason for this phenomenon is that MultiAdam is estimating the momentum of the PDE and boundary objectives
and once it obtains a good estimate, the super-fast convergence rate is observed. In contrast, the other methods converge slowly with much more unstable phenomenons. These results demonstrate the high efficiency and stability of MultiAdam. 

\subsection{Burgers' equation}

The Burgers' equation is a fundamental PDE that describes the evolution of a velocity field in one spatial dimension, represented as follows:
\begin{align}
\begin{aligned}
    u_t+uu_x-\nu u_{xx}&=0,~\forall x\in [-1,1],t\in [0,1] \\
    u(0,x)&=-\sin(\pi x) \\
    u(t,-1)&=u(t,1)=0 ,
\end{aligned}
\end{align}
where $\nu=\frac{0.01}{\pi}$. It can display parabolic or hyperbolic behaviors depending on the relative importance of the forces present.

Table \ref{burgers-error} show the results. We can see that our method has 2.92\% lower error than the baseline PINNs, yet NTK reweighting is even lower in this case. Comparing with NTK reweighting, MultiAdam is more stable, as illustrated in Figure \ref{berger-mape}, where we present the curves of relative $L^2$ error when using Adam, MultiAdam, and NTK methods. We can see that for MultiAdam the error stays at a relatively low position since the middle of training, while Adam's error periodically rises up to as large as 30\%. The spike phenomenon is not so eminent for NTK reweighting, but it is still remarkably worse than MultiAdam's.

\begin{table}[t]
\caption{Mean absolute error and relative $L^2$ error of different optimization methods on Burgers' equation.}
\label{burgers-error}
\vskip 0.15in
\begin{center}
\begin{small}
\begin{tabular}{lll}
\hline
\multirow{2}{*}{Methods} & \multicolumn{2}{c}{Burgers-1} \\
                         & Absolute   & Relative         \\ \hline
Adam                     & 1.61E-02   & 5.87\%           \\
LRA                      & 8.23E-03   & 2.71\%           \\
NTK                      & \textbf{3.47E-03}   & \textbf{1.24\%}           \\
GradNorm                 & 4.81E-03   & 1.51\%           \\
PCGrad                   & 6.18E-02   & 15.96\%  \\ 
MultiAdam                & 5.45E-03   & 2.95\%           \\ \hline
\end{tabular}
\end{small}
\end{center}
\vskip -0.1in
\end{table}

\begin{figure}[ht]
\begin{center}
\centerline{\includegraphics[width=\columnwidth]{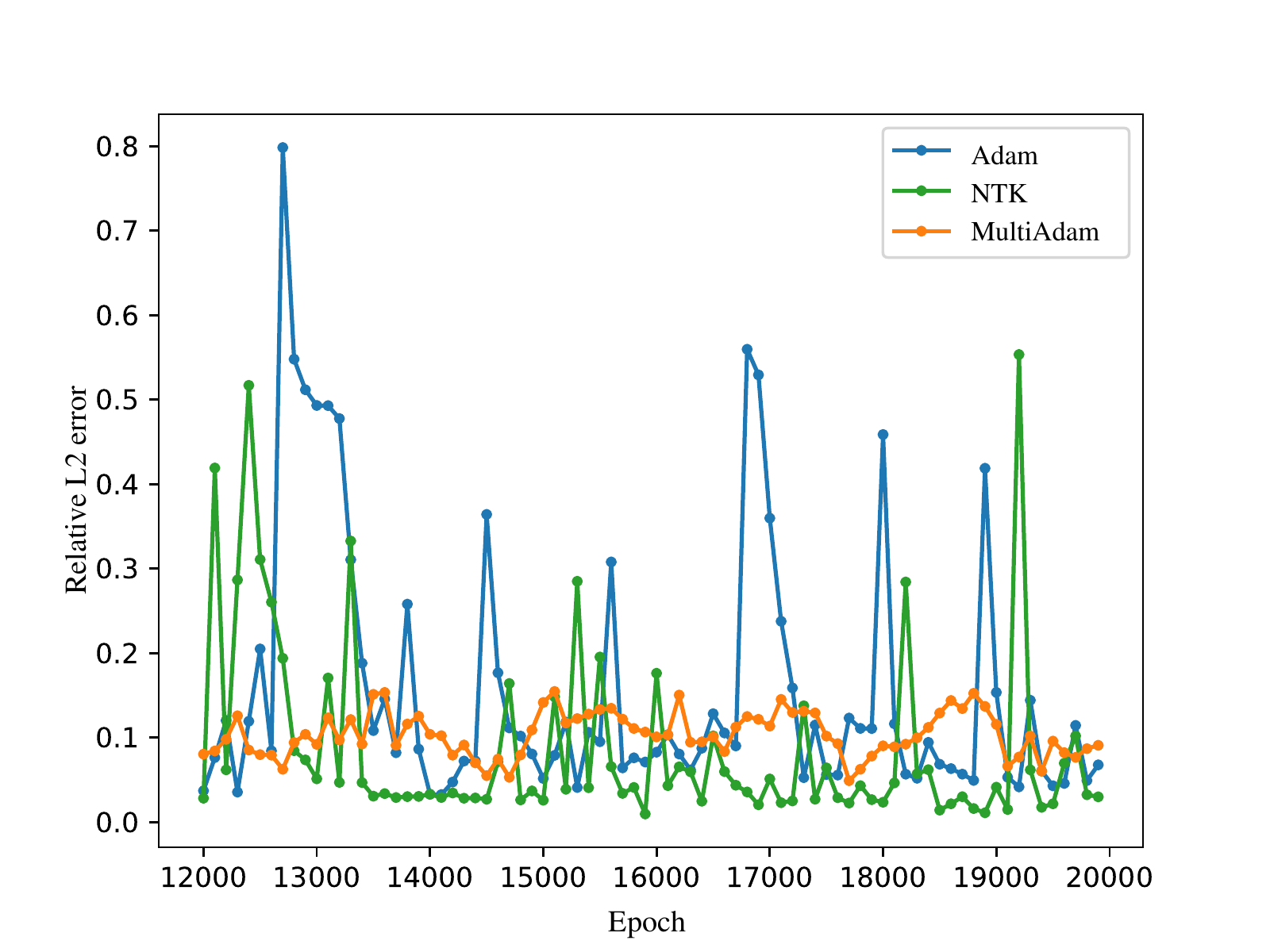}}
\caption{The maximum error across five runs during the last 8000 iterations of training on the Burgers' equation.}
\label{berger-mape}
\end{center}
\vskip -0.2in
\end{figure}

\section{Conclusion}
\label{section:conclusion}

This study primarily aimed to develop a scale-invariant approach for training Physics-Informed Neural Networks (PINNs). We highlighted the impact of domain scaling on PDE loss terms, which significantly contributes to unbalanced losses, and discussed its negative effect on PINN training. To address this issue, we introduced MultiAdam, a parameter-wise scale-invariant optimizer specifically designed for training PINNs. Our numerical experiments demonstrated that this optimizer is capable of handling a variety of cases across different scales, offering a relatively stable training process. 
At the same time, we provided a theoretical analysis of the error bounds of PINNs, which characterize the relationship between the PINN loss terms and the actual performance. 

\section*{Acknowledgements}

This work was supported by the NSF of China Projects (Nos. 62061136001, 61620106010, 62076145, U19B2034, U1811461, U19A2081, 6197222, 62106120, 62076145); a grant from Tsinghua Institute for Guo Qiang; the High Performance Computing Center, Tsinghua University. J.Z was also supported by the New Cornerstone Science Foundation through the XPLORER PRIZE.

\bibliography{reference}

\begin{thebibliography}{34}
\providecommand{\natexlab}[1]{#1}
\providecommand{\url}[1]{\texttt{#1}}
\expandafter\ifx\csname urlstyle\endcsname\relax
  \providecommand{\doi}[1]{doi: #1}\else
  \providecommand{\doi}{doi: \begingroup \urlstyle{rm}\Url}\fi

\bibitem[Bai et~al.(2022)Bai, Rabczuk, Gupta, Alzubaidi, and
  Gu]{bai2022physics}
Bai, J., Rabczuk, T., Gupta, A., Alzubaidi, L., and Gu, Y.
\newblock A physics-informed neural network technique based on a modified loss
  function for computational 2d and 3d solid mechanics.
\newblock \emph{Computational Mechanics}, pp.\  1--20, 2022.

\bibitem[Bathe(2006)]{bathe2006finite}
Bathe, K.-J.
\newblock \emph{Finite element procedures}.
\newblock Klaus-Jurgen Bathe, 2006.

\bibitem[Chen et~al.(2018)Chen, Badrinarayanan, Lee, and
  Rabinovich]{chen2018gradnorm}
Chen, Z., Badrinarayanan, V., Lee, C.-Y., and Rabinovich, A.
\newblock Gradnorm: Gradient normalization for adaptive loss balancing in deep
  multitask networks.
\newblock In \emph{International conference on machine learning}, pp.\
  794--803. PMLR, 2018.

\bibitem[De~Ryck \& Mishra(2022)De~Ryck and Mishra]{de2022error}
De~Ryck, T. and Mishra, S.
\newblock Error analysis for physics-informed neural networks (pinns)
  approximating kolmogorov pdes.
\newblock \emph{Advances in Computational Mathematics}, 48\penalty0
  (6):\penalty0 1--40, 2022.

\bibitem[De~Ryck et~al.(2022)De~Ryck, Jagtap, and Mishra]{de2022error2}
De~Ryck, T., Jagtap, A.~D., and Mishra, S.
\newblock Error estimates for physics informed neural networks approximating
  the navier-stokes equations.
\newblock \emph{arXiv preprint arXiv:2203.09346}, 2022.

\bibitem[Elhamod et~al.(2022)Elhamod, Bu, Singh, Redell, Ghosh, Podolskiy, Lee,
  and Karpatne]{elhamod2022cophy}
Elhamod, M., Bu, J., Singh, C., Redell, M., Ghosh, A., Podolskiy, V., Lee,
  W.-C., and Karpatne, A.
\newblock Cophy-pgnn: Learning physics-guided neural networks with competing
  loss functions for solving eigenvalue problems.
\newblock \emph{ACM Transactions on Intelligent Systems and Technology},
  13\penalty0 (6):\penalty0 1--23, 2022.

\bibitem[Glorot \& Bengio(2010)Glorot and Bengio]{glorot2010understanding}
Glorot, X. and Bengio, Y.
\newblock Understanding the difficulty of training deep feedforward neural
  networks.
\newblock In \emph{Proceedings of the thirteenth international conference on
  artificial intelligence and statistics}, pp.\  249--256. JMLR Workshop and
  Conference Proceedings, 2010.

\bibitem[Grossmann et~al.(2007)Grossmann, Roos, and
  Stynes]{grossmann2007numerical}
Grossmann, C., Roos, H.-G., and Stynes, M.
\newblock \emph{Numerical treatment of partial differential equations}, volume
  154.
\newblock Springer, 2007.

\bibitem[Hao et~al.(2022)Hao, Liu, Zhang, Ying, Feng, Su, and
  Zhu]{hao2022physics}
Hao, Z., Liu, S., Zhang, Y., Ying, C., Feng, Y., Su, H., and Zhu, J.
\newblock Physics-informed machine learning: A survey on problems, methods and
  applications.
\newblock \emph{arXiv preprint arXiv:2211.08064}, 2022.

\bibitem[Herman(2015)]{herman2015introduction}
Herman, R.~L.
\newblock Introduction to partial differential equations.
\newblock \emph{North Carolina, NC, USA: RL Herman}, 2015.

\bibitem[Karniadakis et~al.(2021)Karniadakis, Kevrekidis, Lu, Perdikaris, Wang,
  and Yang]{karniadakis2021physics}
Karniadakis, G.~E., Kevrekidis, I.~G., Lu, L., Perdikaris, P., Wang, S., and
  Yang, L.
\newblock Physics-informed machine learning.
\newblock \emph{Nature Reviews Physics}, 3\penalty0 (6):\penalty0 422--440,
  2021.

\bibitem[Kingma \& Ba(2014)Kingma and Ba]{kingma2014adam}
Kingma, D.~P. and Ba, J.
\newblock Adam: A method for stochastic optimization.
\newblock \emph{arXiv preprint arXiv:1412.6980}, 2014.

\bibitem[Kissas et~al.(2020)Kissas, Yang, Hwuang, Witschey, Detre, and
  Perdikaris]{kissas2020machine}
Kissas, G., Yang, Y., Hwuang, E., Witschey, W.~R., Detre, J.~A., and
  Perdikaris, P.
\newblock Machine learning in cardiovascular flows modeling: Predicting
  arterial blood pressure from non-invasive 4d flow mri data using
  physics-informed neural networks.
\newblock \emph{Computer Methods in Applied Mechanics and Engineering},
  358:\penalty0 112623, 2020.

\bibitem[Li \& Feng(2022)Li and Feng]{li2022dynamic}
Li, S. and Feng, X.
\newblock Dynamic weight strategy of physics-informed neural networks for the
  2d navier--stokes equations.
\newblock \emph{Entropy}, 24\penalty0 (9):\penalty0 1254, 2022.

\bibitem[Liu \& Nocedal(1989)Liu and Nocedal]{liu1989limited}
Liu, D. and Nocedal, J.
\newblock On the limited memory method for large scale optimization:
  Mathematical programming b.
\newblock 1989.

\bibitem[Lu et~al.(2021)Lu, Meng, Mao, and Karniadakis]{lu2021deepxde}
Lu, L., Meng, X., Mao, Z., and Karniadakis, G.~E.
\newblock Deepxde: A deep learning library for solving differential equations.
\newblock \emph{SIAM review}, 63\penalty0 (1):\penalty0 208--228, 2021.

\bibitem[McClenny \& Braga-Neto(2020)McClenny and Braga-Neto]{mcclenny2020self}
McClenny, L. and Braga-Neto, U.
\newblock Self-adaptive physics-informed neural networks using a soft attention
  mechanism.
\newblock \emph{arXiv preprint arXiv:2009.04544}, 2020.

\bibitem[Paszke et~al.(2017)Paszke, Gross, Chintala, Chanan, Yang, DeVito, Lin,
  Desmaison, Antiga, and Lerer]{paszke2017automatic}
Paszke, A., Gross, S., Chintala, S., Chanan, G., Yang, E., DeVito, Z., Lin, Z.,
  Desmaison, A., Antiga, L., and Lerer, A.
\newblock Automatic differentiation in pytorch.
\newblock 2017.

\bibitem[Peng et~al.(2020)Peng, Zhou, Zhang, and Yao]{peng2020accelerating}
Peng, W., Zhou, W., Zhang, J., and Yao, W.
\newblock Accelerating physics-informed neural network training with prior
  dictionaries.
\newblock \emph{arXiv preprint arXiv:2004.08151}, 2020.

\bibitem[Protter \& Weinberger(2012)Protter and Weinberger]{protter2012maximum}
Protter, M.~H. and Weinberger, H.~F.
\newblock \emph{Maximum principles in differential equations}.
\newblock Springer Science \& Business Media, 2012.

\bibitem[Raissi et~al.(2019)Raissi, Perdikaris, and
  Karniadakis]{raissi2019physics}
Raissi, M., Perdikaris, P., and Karniadakis, G.~E.
\newblock Physics-informed neural networks: A deep learning framework for
  solving forward and inverse problems involving nonlinear partial differential
  equations.
\newblock \emph{Journal of Computational physics}, 378:\penalty0 686--707,
  2019.

\bibitem[Raissi et~al.(2020)Raissi, Yazdani, and Karniadakis]{raissi2020hidden}
Raissi, M., Yazdani, A., and Karniadakis, G.~E.
\newblock Hidden fluid mechanics: Learning velocity and pressure fields from
  flow visualizations.
\newblock \emph{Science}, 367\penalty0 (6481):\penalty0 1026--1030, 2020.

\bibitem[Sahli~Costabal et~al.(2020)Sahli~Costabal, Yang, Perdikaris, Hurtado,
  and Kuhl]{sahli2020physics}
Sahli~Costabal, F., Yang, Y., Perdikaris, P., Hurtado, D.~E., and Kuhl, E.
\newblock Physics-informed neural networks for cardiac activation mapping.
\newblock \emph{Frontiers in Physics}, 8:\penalty0 42, 2020.

\bibitem[Shin et~al.(2020)Shin, Zhang, and Karniadakis]{shin2020error}
Shin, Y., Zhang, Z., and Karniadakis, G.~E.
\newblock Error estimates of residual minimization using neural networks for
  linear pdes.
\newblock \emph{arXiv preprint arXiv:2010.08019}, 2020.

\bibitem[Strauss(2007)]{strauss2007partial}
Strauss, W.~A.
\newblock \emph{Partial differential equations: An introduction}.
\newblock John Wiley \& Sons, 2007.

\bibitem[Sun et~al.(2020)Sun, Gao, Pan, and Wang]{sun2020surrogate}
Sun, L., Gao, H., Pan, S., and Wang, J.-X.
\newblock Surrogate modeling for fluid flows based on physics-constrained deep
  learning without simulation data.
\newblock \emph{Computer Methods in Applied Mechanics and Engineering},
  361:\penalty0 112732, 2020.

\bibitem[Wang et~al.(2022{\natexlab{a}})Wang, Li, He, and Wang]{wang2022is}
Wang, C., Li, S., He, D., and Wang, L.
\newblock Is {$L^2$} physics-informed loss always suitable for training
  physics-informed neural network?
\newblock In \emph{Advances in Neural Information Processing Systems},
  2022{\natexlab{a}}.

\bibitem[Wang et~al.(2021)Wang, Teng, and Perdikaris]{wang2021understanding}
Wang, S., Teng, Y., and Perdikaris, P.
\newblock Understanding and mitigating gradient flow pathologies in
  physics-informed neural networks.
\newblock \emph{SIAM Journal on Scientific Computing}, 43\penalty0
  (5):\penalty0 A3055--A3081, 2021.

\bibitem[Wang et~al.(2022{\natexlab{b}})Wang, Yu, and Perdikaris]{wang2022and}
Wang, S., Yu, X., and Perdikaris, P.
\newblock When and why pinns fail to train: A neural tangent kernel
  perspective.
\newblock \emph{Journal of Computational Physics}, 449:\penalty0 110768,
  2022{\natexlab{b}}.

\bibitem[Wight \& Zhao(2020)Wight and Zhao]{wight2020solving}
Wight, C.~L. and Zhao, J.
\newblock Solving allen-cahn and cahn-hilliard equations using the adaptive
  physics informed neural networks.
\newblock \emph{arXiv preprint arXiv:2007.04542}, 2020.

\bibitem[Wikipedia(2023{\natexlab{a}})]{wiki:Helmholtz_equation}
Wikipedia.
\newblock {Helmholtz equation} --- {W}ikipedia{,} the free encyclopedia.
\newblock
  \url{http://en.wikipedia.org/w/index.php?title=Helmholtz\%20equation&oldid=1117741633},
  2023{\natexlab{a}}.
\newblock [Online; accessed 26-January-2023].

\bibitem[Wikipedia(2023{\natexlab{b}})]{wiki:Poisson's_equation}
Wikipedia.
\newblock {Poisson's equation} --- {W}ikipedia{,} the free encyclopedia.
\newblock
  \url{http://en.wikipedia.org/w/index.php?title=Poisson's\%20equation&oldid=1143671910},
  2023{\natexlab{b}}.
\newblock [Online; accessed 31-May-2023].

\bibitem[Yu et~al.(2020)Yu, Kumar, Gupta, Levine, Hausman, and
  Finn]{yu2020gradient}
Yu, T., Kumar, S., Gupta, A., Levine, S., Hausman, K., and Finn, C.
\newblock Gradient surgery for multi-task learning.
\newblock \emph{Advances in Neural Information Processing Systems},
  33:\penalty0 5824--5836, 2020.

\bibitem[Zhu et~al.(2019)Zhu, Zabaras, Koutsourelakis, and
  Perdikaris]{zhu2019physics}
Zhu, Y., Zabaras, N., Koutsourelakis, P.-S., and Perdikaris, P.
\newblock Physics-constrained deep learning for high-dimensional surrogate
  modeling and uncertainty quantification without labeled data.
\newblock \emph{Journal of Computational Physics}, 394:\penalty0 56--81, 2019.

\end{thebibliography}
\bibliographystyle{icml2022}

\newpage
\appendix
\onecolumn

\section{Details of the sample case of Poisson's equation}
\label{detail:poisson}

Its PDE and boundary conditions are as follow:
\begin{align}
\label{eqn:poi8}
\begin{gathered}
\Delta u(x) = 0,\forall x\in \Omega\\
u(x) = 1,\forall x\in B_1\\
u(x) = 0,\forall x\in B_2\\
\end{gathered}
\end{align}

\begin{align}
\begin{aligned}
\Omega &= [-4, 4]^2 \setminus \{x | (x_1\pm 2)^2 + (x_2\pm 2)^2 <1\}\\
B_1 &= \{x|x_1 \in \{-4,4\}\} \cup \{x|x_2\in \{-4,4\}\}\\
B_2 &= \{x|(x_1 \pm 2)^2 + (x_2 \pm 2)^2 =1\}\\
\end{aligned}
\end{align}

\section{Details and Proofs of Theorems}

\subsection{Proof of Theorem \ref{thm:scaling}}
\label{pf:scaling}
\begin{proof}
Assume the operator of the PDE is $\mathcal L$, which is a homogeneous PDE operator of $k$ order. We can decompose the homogeneous $k$ order operator to:
\begin{align}
\begin{aligned}
    \mathcal L = \sum_{l_1+l_2+\cdots +l_n = k} \eta_{l_1,\cdots,l_n}\partial_{x_1}^{l_1}\partial_{x_2}^{l_2}\cdots \partial_{x_n}^{l_n}
\end{aligned}
\end{align}
where $\partial _{x_i}$ is the partial differential operator in $x_i$ direction, $\eta_{l_1,\cdots,l_n} \in \mathbb R$ is coefficient for the term.

Then, we use $\hat u$ to represent the output of PINN and $\hat u'(x) = \hat u(tx)$ for the output of PINN when we narrow the domain by $t$ times. We first investigate the effect of scaling on the derivatives of $\hat u'$. When we apply the term $\partial_{x_1}^{l_1}\partial_{x_2}^{l_2}\cdots \partial _{x_n}^{l_n}$ on $\hat u'(x)$, we can obtain:

\begin{align}
\begin{aligned}
    \partial_{x_1}^{l_1}\partial_{x_2}^{l_2}\cdots \partial _{x_n}^{l_n}\hat u'(x)&=\partial_{x_1}^{l_1}\partial_{x_2}^{l_2}\cdots \partial _{x_n}^{l_n-1}\partial_{x_n}(\hat u(tx))\\
    &=\partial_{x_1}^{l_1}\partial_{x_2}^{l_2}\cdots \partial _{x_n}^{l_n-1} ((\partial _{x_n}\hat u)(tx) \cdot \partial_{x_n}(tx))\\
    &=\partial_{x_1}^{l_1}\partial_{x_2}^{l_2}\cdots \partial _{x_n}^{l_n-1} ((\partial _{x_n}\hat u)(tx) \cdot t)\\
    &=t\cdot \partial_{x_1}^{l_1}\partial_{x_2}^{l_2}\cdots \partial _{x_n}^{l_n-1} (\partial _{x_n}\hat u)(tx)\\
    &= \cdots\\
    &= t^{l_n} \cdot \partial_{x_1}^{l_1}\partial_{x_2}^{l_2}\cdots \partial _{x_{n-1}}^{l_{n-1}} (\partial _{x_n}^{l_n}\hat u)(tx)\\
    &= \cdots\\
    &= t^{l_1+\cdots+l_n} (\partial_{x_1}^{l_1}\partial_{x_2}^{l_2}\cdots \partial _{x_n}^{l_n}\hat u)(tx)
\end{aligned}
\end{align}

Therefore:

\begin{align}
\begin{aligned}
    \mathcal L \hat u'(x)
    &=\sum_{l_1+\cdots+l_n = k} \eta_{l_1,\cdots,l_n}\partial_{x_1}^{l_1}\partial_{x_2}^{l_2}\cdots \partial _{x_n}^{l_n}\hat u'(x)\\
    &=  \sum_{l_1+\cdots+l_n = k}\eta_{l_1,\cdots,l_n}  t^{l_1+\cdots+l_n}(\partial_{x_1}^{l_1}\partial_{x_2}^{l_2}\cdots \partial _{x_n}^{l_n}\hat u)(tx)\\
    &= t^k (\mathcal L\hat u)(tx)
\end{aligned}
\end{align}

Now, we can see the effect of domain scaling on PDE loss as well as boundary loss. The $L^2$ loss before scaling is:

\begin{align}
\begin{aligned}
L_f&=\frac 1{|T_f|} \sum_{x\in T_f} \|\mathcal L \hat u(x)\|_2^2\\
L_b&=\frac 1{|T_b|} \sum_{x\in T_b} \|\hat u(x) - B(x)\|_2^2\\
\end{aligned}
\end{align}
Where $B(x)$ is the boundary condition.

Now, if we narrow the domain by $t$ times, the new loss function $L_f', L_b'$ will be:

\begin{align}
\begin{aligned}
L_f' &= \frac1{|T_f|} \sum_{x\in T_f} \|\mathcal L \hat u'(x)\|_2^2 \\
&= \frac 1{|T_f|} \sum_{x\in T_f} \| t^k \cdot (\mathcal L \hat u)(tx)\|_2^2\\
&= t^{2k} \frac 1{|T_f|} \sum _{x\in T_f} \| (\mathcal L \hat u)(tx)\|_2^2 \\
&= t^{2k}  L_f\\
L_b' &= \frac 1{|T_b|} \sum_{x\in T_b} \| \hat u'(x)-B'(x)\|_2^2\\
&= \frac 1{|T_b|} \sum_{x\in T_b} \|u(tx) - B(tx)\|_2^2 \\
&= L_b
\end{aligned}
\end{align}
Which leads to the conclusion of the theorem.
\end{proof}

\subsection{Proof of Theorem \ref{thm:parabolic-2}}
\label{pf:parabolic-2}

In the following proof, we denote the parabolic operator as $\mathcal {L}$. It can be formalized as:

\begin{equation}
        \mathcal L[u]=\sum_{1\le i,j\le d}a_{i,j}(x,t)\frac{\partial ^2 u}{\partial x_i \partial x_j}+\sum _{1\le i\le d}b_i(x,t)\frac{\partial u}{\partial x_i}+c(x)u-\frac{\partial u}{\partial t}
\end{equation}

Where $a_{i,j}, b_i, c \in C(\Omega)$ are the coefficient for the parabolic operator.

Firstly, we have to cite the following lemma:

\begin{lemma}[Maximum principle for Parabolic PDEs]
\label{thm:parabolic-1}
Suppose $\Omega \subset \mathbb R_x^d \times \mathbb R_t^+$ and $\mathcal L$ is a parabolic operator defined on $\Omega$. If $\mathcal L[u]\ge 0, c\le 0,\forall x\in \Omega$ and $u\in C^2(\Omega)\cap C^0(\overline \Omega)$ reaches maximum $M\ge 0$ in the interior of $\Omega$, then $\sup _{x\in\partial \Omega} u \ge M$. 
\end{lemma}
\begin{proof}
Thanks to the Theorem 7 in Chap.3 Sec.3 in \cite{protter2012maximum}, we can get the proof of the theorem.
\end{proof}

Now we can start the proof for Theorem \ref{thm:parabolic-2}
\begin{proof}
We first assume that $c\le 0$ holds in $\Omega$, and define $h_1 = u_\theta - \tilde u, h_2 = \mathcal L[u_\theta]-f$. Due to the linearity of operator $\mathcal {L}$, we can see that $\mathcal L[h_1]=h_2$

Since $\Omega$ is bounded, we assume that $\Omega$ lies in the slab $0\le x_1\le d$, and set $\mathcal L_0[u]=\sum a_{i,j}(x,t)\frac{\partial ^2 u}{\partial x_i \partial x_j}+\sum b_i(x,t)\frac{\partial u}{\partial x_i}-\frac{\partial u}{\partial t}$. 

In addition, according to the definition of parabolic differential operators, the coefficient matrix $A(x, t) = \{a_{i,j}(x,t)\}$ is positive definite. Thus, we can define $\lambda(x, t)>0$ as the smallest eigenvalue of $A(x, t)$.

We define $\beta = \sup _{\Omega}(|\mathbf b|/\lambda)$ (well defined due to the maximum principle of continuous function in $\overline \Omega$). So for $\forall \alpha \ge \beta+1$ we have:

\begin{equation}
\mathcal L_0e^{\alpha x_1} = (\alpha ^2a_{1,1}+\alpha b_1)e^{\alpha x_1}\ge \lambda (\alpha ^2 -\alpha \beta )e^{\alpha x_1}\ge \lambda
\end{equation}
Since $a_{1,1}\ge \lambda$ always holds for positive definite matrices. We denote $h_1^+$ as the positive part of $h_1$ and $h_2^-$ for the negative part of $h_2$. Let
\begin{equation}
    v=  \sup _{\partial \Omega} h_1^+ + (e^{\alpha d}-e^{\alpha x_1})\sup _{\Omega}\frac{|h_2^-|}{\lambda}
\end{equation}

Then $v\ge 0$ always holds because $x_1\le d$ and the two parts are always positive. 

Thus,  
\begin{equation}
\mathcal Lv = \mathcal L_0 v+cv\le \mathcal L_0 v = -\sup _{\Omega}\frac{|h_2^-|}{\lambda }\mathcal L_0[e^{\alpha x_1}]\le -\lambda \sup _{\Omega}\frac{|h_2^-|}{\lambda}\\
\end{equation}
Therefore,
$$
\mathcal L(v-h_1)\le -\lambda \sup _{\Omega}\frac{|h_2^-|}{\lambda}-\mathcal L[h_1] = -\lambda \sup _{\Omega}\frac{|h_2^-|}{\lambda}-h_2=-\lambda (\sup _{\Omega}\frac{|h_2^-|}{\lambda}+\frac{h_2}\lambda)\le 0
$$
By Lemma \ref{thm:parabolic-1}, we can get that $h_1\le v$ always holds in $\Omega$. Thus, we have:
$$
\sup _{\Omega} h_1 \le \sup_{\Omega} v \le \sup _{\partial \Omega} h_1^+ + (e^{\alpha d}-1)\sup_{\Omega}\frac{|h_2^-|}{\lambda}
$$
Replacing $h_1,h_2$ by $-h_1,-h_2$, we obtain:
$$
\sup _{\Omega} |h_1|\le \sup _{\partial \Omega}|h_1| + (e^{\alpha d}-1)\sup_\Omega \frac{|h_2|}{\lambda}\le \delta_1 + (e^{\alpha d}-1)\frac{\delta_2}{\lambda_0}
$$

Then, we consider the situation that $c>0$:

Since there exist $\eta >0$ satisfies that $c\le \eta$ always holds on $\overline \Omega$, then we define $\tilde h_1 = e^{-\eta t}h_1$. Then $(\mathcal L -\eta)[\tilde h_1] = (\mathcal Le^{-\eta t})h_1+e^{-\eta t}(\mathcal Lh_1)-\eta e^{-\eta t}h_1=e^{-\eta t}(\mathcal Lh_1)=e^{-\eta t}h_2$. And the operator $\mathcal L-\eta=\sum a_{i,j}D_{i,j}+\sum b_i D_i+c-\eta -D_t$ satisfies $c-\eta \le 0$. Therefore, using the conclusion above we can get that:
$$
\sup _{\Omega} |\tilde h_1|\le \sup _{\partial \Omega} |\tilde h_1|+(e^{\alpha d}-1)\sup _{\Omega}\frac{|e^{-\lambda t}h_2|}{\lambda }\\
$$
Subsequently, if we assume that $\Omega$ lies in $t_0\le t\le t_0+T$, then we can get
\begin{align}
\begin{aligned}
\sup _{\Omega}|h_1|\le& e^{\eta T}(\sup _{\partial \Omega} |h_1|+(e^{\alpha d}-1)\sup _{\Omega}\frac{|h_2|}{\lambda})\\\le& e^{\eta T}(\delta_1+(e^{\alpha d}-1)\frac{\delta_2}{\lambda _0})
\end{aligned}
\end{align}

By setting $C_1 = e^{\eta T}, C = (e^{\alpha d} - 1) / \lambda_0$, we can get the proof of the theorem. 
\end{proof}

\subsection{Details of Theorem \ref{thm:hyperbolic-2}}
\label{pf:hyperbolic-2}

In the following proof, we denote the hyperbolic operator as $\mathcal {L}$, which can be formalized as:
\begin{equation}
    \mathcal L[u]=a(x,t)\frac{\partial ^2 u}{\partial^2 x} + 2b(x,t)\frac{\partial ^2 u}{\partial x \partial t} + c(x,t)\frac{\partial ^2 u}{\partial ^2t} + d(x,t)\frac{\partial  u}{\partial x } + e(x,t)\frac{\partial u}{ \partial t} + h(x,t)u
\end{equation}

Where $a, b, c, d, e, h\in C(\Omega)$ are coefficient functions for the equation.

We will first give definition to Admissible Domain and specify the conditions for the hyperbolic operator $\mathcal {L}$ in Theorem \ref{thm:hyperbolic-2}. Using this assumptions, we will cite a lemma proved in \cite{protter2012maximum} and finally use this lemma to prove Theorem \ref{thm:hyperbolic-2}.

\subsubsection{Definition of Admissible domain}
\label{def:admissible domain}
\begin{definition}
\label{def:characteristic curves} 
(characteristic curves):

The definition can also be found in \cite{protter2012maximum}

For every point $(x,t)\in \Omega$ that satisfies $c(x,t)\ne 0$, we have two characteristic curves, which are the solutions to the following ordinary differential equation:
$$
c(\frac{\text d x}{\text d t})^2 - 2b\frac{\text d x}{\text d t}+a=0
$$
Solving for $\text dx/\text dt$, we have:
$$
\frac{\text dx}{\text dt}=\frac{-b\pm\sqrt {b^2-ac}}{-c}
$$
Thus, we can have two characteristics $C_+$ and $C_-$, corresponding to the two signs in front of the square root.
\end{definition}
\begin{definition}
\label{def:characteristic triangle}
(characteristic triangle):

The definition can also be found in \cite{protter2012maximum}

We assume that $c\le c_0<0$ and for every point $C=(x,t)\in \Omega$, we construct two characteristic curves $C_+,C_-$. We denote by $A$ the point where $C_+$ hits the $x$-axis and by $B$ the point where $C_-$ curve hits it. Then the segment $AB$ and two curves $AC$ and $BC$ form a characteristic triangle $ABC$.

\end{definition}
\begin{definition}
(admissible domain):

The definition can also be found in \cite{protter2012maximum}

A domain $\Omega \subset \mathbb{R}_x\times \mathbb{R}_{t}^+$ is called an admissible domain if it has the property that for every point $C=(x,t)\in \Omega$, the corresponding characteristic triangle $ABC$ with $AB$ on the $x$-axis is also in $\Omega$.
\end{definition}

\subsubsection{Conditions for the hyperbolic operator $\mathcal {L}$ in Theorem \ref{thm:hyperbolic-2}}

The condition for the hyperbolic operator is:
\begin{align}\label{equ:hyperbolic-constraint}
\begin{aligned}
&c\le c_0<0\\
&K_{\pm}\ge 0\\
&\frac{\partial ^2 a}{\partial^2 x} + 2\frac{\partial ^2 b}{\partial x \partial t} + \frac{\partial ^2 c}{\partial ^2t} - \frac{\partial d}{\partial x } - \frac{\partial e}{ \partial t} + h\ge 0
\end{aligned}
\end{align}
where $c_0$ is a negative constant. $K_\pm$ are:

\begin{align}
\begin{aligned}
K_\pm=&\frac{\partial}{\partial t}(\sqrt{b^2-ac})+\frac bc\frac{\partial}{\partial x}(\sqrt{b^2-ac})+\frac1c(\frac{\partial b}{\partial x}+\frac{\partial c}{\partial t}-e)\sqrt{b^2-ac}\\&\mathop{\pm}\left[-\frac1{2c}\frac{\partial}{\partial x}(b^2-ac)+\frac{\partial a}{\partial x}+\frac{\partial b}{\partial t}-d-\frac bc (\frac{\partial b}{\partial x}+\frac{\partial c}{\partial t}-e)\right]
\end{aligned}
\end{align}

\subsubsection{Proof of Theorem \ref{thm:hyperbolic-2}}

In order to proof the theorem, we first introduce a Lemma in \cite{protter2012maximum}:

\begin{lemma}[Maximum principle for Hyperbolic PDEs]
\label{thm:hyperbolic-1}
Suppose $\Omega \subset \mathbb{R}_x\times \mathbb{R}_t^+$ is an admissible domain and denote $\Gamma_0 = \Omega \cap \{t=0\}$. Assume the operator $\mathcal L$ satisfies the constraints in Equation (\ref{equ:hyperbolic-constraint}). Then, if a function $u\in C^2(\Omega)\cap C^1(\Omega\cup \Gamma_0)$ satisfies:
\begin{align}
\begin{aligned}
&\mathcal L[u]\ge 0\ \ \forall (x,t)\in \Omega\\
&u\le 0\ \ \forall (x,t)\in \Gamma_0\\
&b\frac{\partial u}{\partial x}+c\frac{\partial u}{\partial t}-(\frac{\partial b}{\partial x}+\frac{\partial c}{\partial t}-e)u\ge 0\ \ \forall (x,t)\in \Gamma_0
\end{aligned}
\end{align}

Then $u\le 0$ in $\Omega$. 

Moreover, if we replace $u\le 0$ in $\Gamma_0$ by $u\le M$, where $M\ge 0$ is a constant, and add two constraints that $h\le 0$ and $\frac{\partial b}{\partial x}+\frac{\partial c}{\partial t}-e\ge 0, \forall (x,t)\in \Gamma_0$, then $u\le M$ in $\Omega$. 
\end{lemma}

\begin{proof}
Detailed proof can be found in Chap.4 Sec.3 in \cite{protter2012maximum}
\end{proof}

Now, we can start the proof of our theorem.

\begin{proof}
Firstly, we assume that $u_\theta = \tilde u$ on the boundary $\Gamma_0$.

We define $h_1 = u_\theta - \tilde u, h_2 = \mathcal L[u_\theta]-f$. Due to the linearity of hyperbolic operator $\mathcal L$, we can see that $\mathcal L[h_1]=h_2$ and $h_1=0$ on $\Gamma_0$.

Define $v =  (e^{\alpha t}-1)\sup _{\Omega}|h_2^-|$, Here, $\alpha = \frac{N+\sqrt{N^2-4c_0}}{-2c_0}>0$ where $N=\sup _{\Omega}|e|$

Therefore, by the definition of $\alpha$, we have $\alpha^2 c+\alpha e\le \alpha^2 c_0+\alpha M+1-1\le -1$, so:
$$
\mathcal L(e^{\alpha t})=(\alpha ^2c+\alpha e)e^{\alpha t}\le -e^{\alpha t}
$$
According to the assumption that $h\le 0$ and the definition of $v$:
\begin{align}
\begin{aligned}
(L+h)(v) &= (Le^{\alpha t}) \sup_\Omega|h_2^{-}|+hv\le -e^{\alpha t}\sup _\Omega |h_2^-|\\&\le -\sup _\Omega |h_2^-|
\end{aligned}
\end{align}
Therefore,
$$
(L+h)(h_1-v)\ge h_2+\sup _\Omega |h_2^-|\ge 0
$$
At the same time, given that $h_1=v=0$ on $\Gamma_0$, we have:
\begin{align}
\begin{aligned}
-b\frac{\partial (h_1-v)}{\partial x}-c\frac{\partial (h_1-v)}{\partial t}+(b_x+c_t-e)(u-v)\\=c\frac{\partial v}{\partial t}=c\alpha e^{\alpha t}\sup _\Omega |f^-|\le 0
\end{aligned}
\end{align}
Thus, according to Theorem \ref{thm:hyperbolic-1}, $h_1\le v$ holds on $\Omega$, so $\sup h_1\le (e^{\alpha T}-1)\sup _\Omega |h_2^-|$, where $T$ is the upper bound of $t$-coordinate of the points in $\Omega$.

By replacing $h_1,h_2$ with $-h_1,-h_2$, we have $\sup |h_1|\le (e^{\alpha T}-1)\delta_2$.

Now, we consider $u_\theta \ne \tilde u$ on boundary $\Gamma_0$

Let $w$ be the solution to the following PDE:

\begin{align}
\begin{aligned}
\mathcal L[w](x) = f(x),\ \  &\forall x\in \Omega\\
w(x) =u_\theta(x), \ \ &\forall x\in \partial \Omega
\end{aligned}
\end{align}

So $h_1=u_\theta-w$ and $h_2 = \mathcal L[u_\theta] - f$ satisfies the conditions above. so $\|u_\theta-w\|_{L^\infty}\le (e^{\alpha T}-1)\delta_2$

Also, $h_3=\tilde u-w$ satisfies that $\mathcal L[h_3]=0$ and $h_3\le \sup _{\Gamma_0}|u_\theta - \tilde u|<\delta_1$, so by Theorem \ref{thm:hyperbolic-1}, $\sup _{\Omega}h_3 \le \delta_1$

Replacing $h_3$ by $-h_3$, we have $\|\tilde u-w\|_{L^\infty}\le \delta_1$, so finally $\|u_\theta - \tilde u\|_{L^\infty}\le \delta_1+(e^{\alpha T}-1)\delta_2$. Here, taking $C=e^{\alpha T}-1$ is a possible value.
\end{proof}

\subsection{Proof of Theorem \ref{thm:L2control}}
\label{pf:L2control}
\begin{proof}

We first define the error in the domain: $h_1 = u_\theta - \tilde u, h_2 = \mathcal L[u_\theta] - f$, and we assume the error can be approximated by a set of base function $\{\phi_i\}$. That is:

\begin{align}
\begin{aligned}
    h_1(x) = (1 + \varepsilon_1(x))\sum_{i = 1}^n \lambda_i \phi_i(x) \\
    h_2(x) = (1 + \varepsilon_2(x))\sum_{i = 1}^n \eta_i \phi_i(x) 
\end{aligned}
\end{align}

Where $\lambda_i, \eta_i$ are coefficients and $\varepsilon_1, \varepsilon_2$ are the \textit{relative} errors in the approximation. A proper set of the base function can be produced by FEM methods in $\Omega$. 

Then, according to the definition in Equation (\ref{equ:pinn-loss}), we have:
\begin{align}
\begin{aligned}
L_b &= \frac {1}{|T_b|} \sum _{x\in T_b} |h_1(x)|^2\\
&= \frac 1{|T_b|} \sum_{x\in T_b}\left ((1 + \varepsilon_1(x))\sum_{i=1}^n \lambda_i\phi_i(x)\right )^2 \\
&\ge  \frac 1{|T_b|} \sum_{x\in T_b}(\sum_{i=1}^n \lambda_i\phi_i(x))^2 \cdot \inf_{x\in \partial \Omega} (1 + \varepsilon_1(x))^2\\
&\ge \sum_{i=1}^n\sum_{j=1}^n \lambda_i \lambda_j \left (\frac 1{|T_b|} \sum_{x\in T_b}\phi_i(x)\phi_j(x)\right )\cdot  (1 - \|\varepsilon_1\|_{L_\infty})^2\\
L_f &= \frac {1}{|T_b|} \sum _{x\in T_b} |h_2(x)|^2\\
&= \frac 1{|T_f|} \sum_{x\in T_f}\left ((1 + \varepsilon_2(x))\sum_{i=1}^n \eta_i\phi_i(x)\right )^2 \\
&\ge  \frac 1{|T_f|} \sum_{x\in T_f}(\sum_{i=1}^n \eta_i\phi_i(x))^2 \cdot \inf_{x\in \Omega} (1 + \varepsilon_2(x))^2\\
&\ge \sum_{i=1}^n\sum_{j=1}^n \eta_i \eta_j \left (\frac 1{|T_f|} \sum_{x\in T_f}\phi_i(x)\phi_j(x)\right )\cdot  (1 - \|\varepsilon_2\|_{L_\infty})^2\\
\end{aligned}
\end{align}

We denote $a_{i,j} = \frac 1{|T_b|} \sum_{x\in T_b}\phi_i(x)\phi_j(x)$ and it can construct a matrix $A = \{a_{i,j}\}$. The matrix is positive definite since it is the metric matrix for the space $V = \mathop{\mathrm{span}} \{\phi_i\}$ equipped with the inner product $(f, g) = \frac 1{|T_b|} \sum_{x\in T_b}f(x)g(x)$. Similarly, when we define $b_{i,j} = \frac 1{|T_f|} \sum_{x\in T_f}\phi_i(x)\phi_j(x)$, it also constructs a positive definite matrix $B = \{b_{i,j}\}$.

Therefore:

\begin{align}\label{equ:inequ-for-Ls}
\begin{aligned}
L_b & \ge \mathbb {\lambda} ^T A \mathbf {\lambda} \cdot (1 - \|\varepsilon_1\|_{L_\infty})^2\\
L_f & \ge \mathbf {\eta} ^T B \mathbf {\eta} \cdot (1 - \|\varepsilon_2\|_{L_\infty})^2\\
\end{aligned}
\end{align}

Moreover, we have:
\begin{align}\label{equ:inequ-for-hs}
\begin{aligned}
\sup_{x\in \partial \Omega} |h_1| &\le \sup_{x\in \partial \Omega} \sum_{i=1}^n (1+\varepsilon_1(x)) |\lambda_i \phi_i(x)|\\
&\le (1 + \|\varepsilon_1\|_{L_\infty}) \sum_{i=1}^n |\lambda_i| \sup_{x\in \partial \Omega} |\phi_i(x)|\\
\sup_{\Omega} |h_2| &\le \sup_{x\in \Omega} \sum_{i=1}^n (1+\varepsilon_2(x)) |\eta_i \phi_i(x)|\\
&\le (1+\|\varepsilon_2\|_{L_\infty}) \sum_{i=1}^n |\eta_i| \sup_{x\in  \Omega} |\phi_i(x)|\\
\end{aligned}
\end{align}

Now, we define (we assume the approximation error $\|\varepsilon_1\|_{L_\infty} ,\|\varepsilon_1\|_{L_\infty}$ are close to $0$ ):

\begin{align}\label{equ:def-of-D}
\begin{aligned}
D_1 &= \sup_{\|\lambda\| = 1} \frac{(1 + \|\varepsilon_1\|_{L_\infty}) \sum _{i=1}^n |\lambda_i|\sup_{x\in \partial \Omega} |\phi_i(x)|} {(1 - \|\varepsilon_1\|_{L_\infty}) \sqrt{\lambda^T A \lambda}} \\
D_2 &= \sup_{\|\eta\| = 1} \frac{(1 + \|\varepsilon_2\|_{L_\infty}) \sum _{i=1}^n |\eta_i|\sup_{x\in \Omega} |\phi_i(x)|} {(1 - \|\varepsilon_2\|_{L_\infty}) \sqrt{\eta^T B \eta}} \\
\end{aligned}
\end{align}

$D_1$ is well posed since $A, B$ are positive definite, thus the function is bounded in the compact domain $\|\lambda\|=1$. The well-posedness for $D_2$ is similar.

Finally, by combining Equation (\ref{equ:inequ-for-Ls}), (\ref{equ:inequ-for-hs}) and (\ref{equ:def-of-D}), we have:
\begin{align}
\begin{aligned}
\sup_{x\in \partial \Omega} |h_1| & \le D_1 \sqrt{L_b}\\
\sup_{x\in \Omega} |h_2| & \le D_2 \sqrt{L_f}\\
\end{aligned}
\end{align}

Given the condition in the theorem \ref{thm:L2control}, we have:
\begin{align}
\begin{aligned}
    \|u_\theta - \tilde u\|_{L_\infty} &\le C_1(\sup_{x\in \partial \Omega} |u_\theta - \tilde u| + C \sup_{x\in \Omega} |\mathcal L[u_\theta] - f|)\\
    & = C_1(\sup_{x\in \partial \Omega} |h_1| + C \sup_{x\in \Omega} |h_2|)\\
    & \le \max\{D_1, C\cdot D_2\}C_1(\sqrt{L_b} + C \sqrt{L_f})\\
\end{aligned}
\end{align}

By setting $C_2 = \max\{D_1, C\cdot D_2\}C_1$, we can get the proof of our theorem.

\end{proof}

\subsection{Proof of Theorem \ref{thm:laplaceerror}}
\label{pf:laplaceerror}
\begin{proof}

First, let $\Omega \subset \mathbb R^2$ and $f\in C^\infty (\Omega), g\in C^\infty(\partial \Omega)$. The Poisson's Equation is:
\begin{align}
\begin{aligned}
-\Delta u& = f, \forall x\in \Omega\\
u &= g, \forall x \in \partial \Omega
\end{aligned}
\end{align}
We denote the solution to the equation above as $\tilde u$.

Assume the PINN estimation is $u_\theta$ and the error is $v = u - u_\theta$, we will have:
\begin{align}
\begin{aligned}
-\Delta v &= f + \Delta u_\theta, \forall x\in \Omega\\
v& = g - u_\theta, \forall x \in \partial \Omega
\end{aligned}
\end{align}
and we can explicitly write the error with the help of Green function $G(x,\xi)$ (See chapter 7.5 of \cite{herman2015introduction} )
\begin{align}
\begin{aligned}
v(x) = \int_{\Omega} (f(\xi) + \Delta u_\theta(\xi))G(x,\xi) \text d \xi - \int _{\partial\Omega} (g(\xi) - u_\theta(\xi))\nabla_{\xi}G(x, \xi) \cdot \mathbf{n} \text d \sigma
\end{aligned}
\end{align}
where $\nabla _{\xi} G(x,\xi)$ is the gradient of $G(x,\xi)$ with respect to $\xi$ and $\mathbf n$ is the normal vector of $\partial \Omega$.

Then, we use Cauchy inequality:

\begin{align}
\begin{aligned}
 &\int_{\Omega} (f(\xi) + \Delta u_\theta(\xi))G(x,\xi) \text d \xi \\
 \le & \sqrt{ \int_{\Omega} (f(\xi) + \Delta u_\theta(\xi))^2 \text d \xi}\sqrt{\int _{\Omega}G^2(x,\xi) \text d \xi}\\
 &\int _{\partial\Omega} (g(\xi) - u_\theta(\xi))\nabla_{\xi}G(x, \xi) \cdot \mathbf{n} \text d \sigma\\
 \le & \sqrt{\int _{\partial\Omega} (g(\xi) - u_\theta(\xi))^2 \text d \xi }\sqrt{\int_{\partial\Omega}(\nabla_{\xi}G(x, \xi) \cdot \mathbf{n})^2 \text d \sigma}
\end{aligned}
\end{align}

So, if we define our model's $L^2$ loss as follows:

\begin{align}
\begin{aligned}
L_{f} = \frac 1{|\Omega|}\int_{\Omega} (f(\xi) + \Delta u_\theta(\xi))^2 \text d \xi\\
L_{b} = \frac 1{|\partial\Omega|} \int _{\partial\Omega} (g(\xi) - u_\theta(\xi))^2 \text d \xi
\end{aligned}
\end{align}

we can get the control of the error:
\begin{align}
\begin{aligned}
|v(x)| &\le \sqrt{L_{f} \cdot |\Omega|} \sqrt{\int _{\Omega}G^2(x,\xi) \text d \xi}\\& + \sqrt{L_{b} \cdot |\partial\Omega|} \sqrt{\int_{\partial\Omega}(\nabla_{\xi}G(x, \xi) \cdot \mathbf{n})^2 \text d \sigma}
\end{aligned}
\end{align}

Finally, we can get the following conclusion:

\begin{align}
\begin{aligned}
\|v\|_{L^1} \le C_1 \sqrt{L_{f}} + C_2\sqrt{L_{b}}
\end{aligned}
\end{align}

by defining $C_1 = \int _{\Omega} \sqrt{|\Omega|\int _{\Omega}G^2(x,\xi) \text d \xi} \text d x,~C_2 = \int _{\Omega} \sqrt{|\partial \Omega|\int_{\partial\Omega}(\nabla_{\xi}G(x, \xi) \cdot \mathbf{n})^2} \text d x$

\end{proof}


\section{Details of Experiments}
\label{app:exp}

We run the experiments based on DeepXDE 1.6.1 \cite{lu2021deepxde} with Pytorch 1.9 \cite{paszke2017automatic} backend. We use the default hyper-parameters for all the methods. The code will be released at \url{https://github.com/i207M/MultiAdam}

We strictly control the non-experimental variables of tests to be the same. In all examples, we use a five-layer feed-forward network of width 100 as the base model. The training dataset contains 10000 random points sampled from the domain and 1000 from boundaries. 

The accuracy of the methods is measured by mean absolute error (MAE) and relative $L^2$ error, which is explained in Appendix \ref{app:measure}. To reduce randomness, we repeat every setup 5 times with the Glorot normal initializer \cite{glorot2010understanding} and provide their average.

\subsection{Measurement}
\label{app:measure}

The metrics we use are mean absolute error and relative $L^2$ error as follows:

\begin{equation}
    \text{relative}~L_2~\text{error}=\frac{\sqrt{\sum_{i=1}^N|\hat u(x_i,t_i)-u(x_i,t_i)|^2}}{\sqrt{\sum_{i=1}^N|u(x_i,t_i)|^2}}
\end{equation}

where $u$ is the exact solution and $\hat u$ is the trained approximation. In cases that $u$ cannot be analytically represented, we utilize the finite element method to obtain high-precision numerical reference.

\subsection{Poisson's equation}

We use $\tanh$ as the activation function and train the five-layer network for 15000 epochs. 

\subsection{Helmholtz equation}
\label{app:hel}

On training hyper-parameters, the activation function of the model is $\sin$ and the number of training epochs is 20000.

\begin{figure*}[ht]
\vskip 0.2in
\begin{center}
\centerline{\includegraphics[width=0.9\columnwidth]{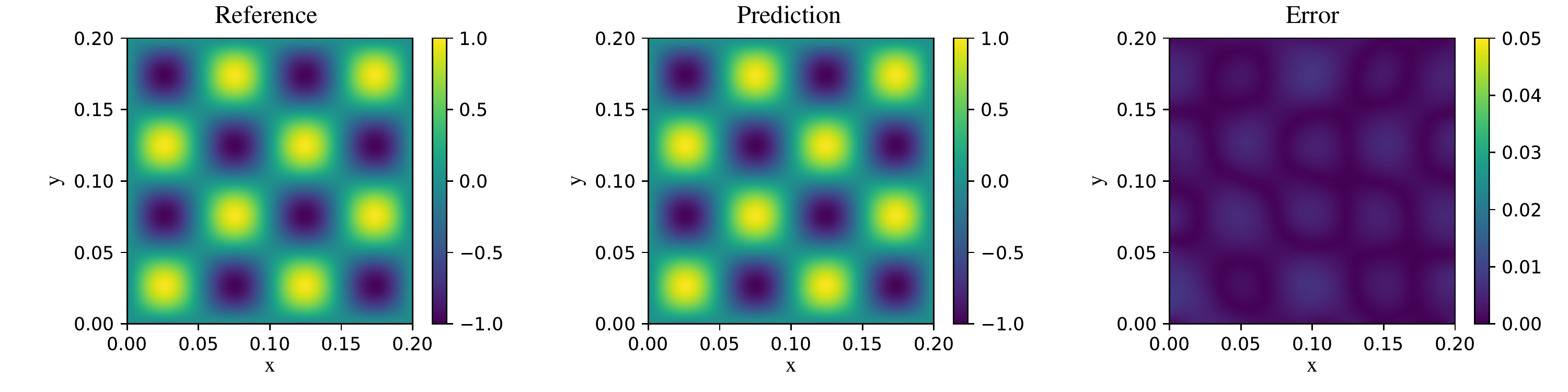}}
\caption{The predicted solution versus the exact solution of Helmholtz Equation by training a five-layer neural network using MultiAdam after 15000 iterations.}
\label{hel-heatmap}
\end{center}
\vskip -0.2in
\end{figure*}

\subsection{Burgers' equation}
\label{app:burgers}

Regarding network settings, $\tanh$ is set as the activation function and we iterate the optimization 20000 times. 

\section{Ablation study on $\beta_1, \beta_2$ hyper-parameters}
\label{ablation}

Here we use examples to demonstrate our interesting findings on the performance impact of betas, often-ignored hyper-parameters of the Adam optimizer. Adam's $\beta_1, \beta_2$ are $(0.9,0.999)$ by default, which may not be optimal for MultiAdam. We compared different settings of betas to illustrate the effect of first-order and second-order momentum estimation in our method. Results are listed in Table \ref{ablation-betas}.

We found that $(0.99,0.99)$ achieves the best convergence, which holds true among other PDE systems after substantial experiments. We argue that the scale invariant ability is related to the equality of $\beta_1$ and $\beta_2$. The equality implies that the optimizer tracks the same period of history of first-order and second-order gradient momentums, so by dividing one by another the scaling factor is eliminated. Hence the Adam's default $(0.9,0.999)$ works badly while $(0.9,0.9)$ is OK. We also tested removing first-order or second-order momentum, which performed so poorly that even encountered numerical instability. Therefore the momentum does help with optimization.

\begin{table}[t]
\caption{Mean absolute error and relative $L^2$ error of different $\beta_1, \beta_2$ settings on Poisson's equation. $(0.9,0)$ runs into NaN multiple times.}
\label{ablation-betas}
\vskip 0.15in
\begin{center}
\begin{small}
\begin{tabular}{lllll}
\hline
\multirow{2}{*}{$\beta_1, \beta_2$} & \multicolumn{2}{c}{Poisson-8} & \multicolumn{2}{c}{Poisson-1} \\
                   & Absolute   & Relative         & Absolute   & Relative         \\ \hline
0.99,0.99          & \textbf{1.10E-02}   & \textbf{2.94\%}           & \textbf{1.44E-02}   & \textbf{4.49\%}  \\
0.9,0.999          & 2.98E-02   & 8.44\%           & 3.04E-01   & 70.19\%          \\
0.9,0.9            & 2.54E-02   & 7.78\%           & 8.25E-02   & 21.09\%          \\
0.9,0            & N/A   & N/A           & N/A   & N/A          \\
0,0.9             & 8.76E-02   & 22.33\%          & 2.92E-01   & 68.31\%          \\ \hline
\end{tabular}
\end{small}
\end{center}
\vskip -0.1in
\end{table}

\end{document}